%% file: emnlp_submission.tex
\documentclass[11pt]{article}

\usepackage[final]{acl}

\usepackage{times}
\usepackage{latexsym}
\usepackage[T1]{fontenc}
\usepackage[utf8]{inputenc}
\usepackage{microtype}
\usepackage{inconsolata}

\usepackage{graphicx}
\usepackage{amsmath,amssymb,bm}
\usepackage{amsfonts}
\usepackage{amsthm}
\usepackage{algorithm}
\usepackage{algorithmic}
\usepackage{booktabs}
\usepackage{adjustbox}
\usepackage{subcaption}
\usepackage{tabularx}
\usepackage{float}
\usepackage{placeins}
\usepackage[noabbrev,capitalise,nameinlink]{cleveref}
\usepackage{xcolor}
\newtheorem{proposition}{Proposition}
\theoremstyle{definition}

\theoremstyle{remark}
\newtheorem*{remark}{Remark}

\newcommand{\tr}{\text{Tr}}

\newcommand{\best}[1]{\textbf{#1}}
\newcommand{\secondbest}[1]{\underline{#1}}

\title{Aligning the Query Space: Greedy Information Projection for Language Model Data Selection}

\author{
   \textbf{Victor Ye Dong\textsuperscript{*,\textdagger}} \quad
   \textbf{Kuan-Yun Lee\textsuperscript{*}} \quad
   \textbf{Jiamei Shuai} \quad 
   \textbf{Shengfei Liu} \quad
   \textbf{Yi Liu} \quad
   \textbf{Jian Jiao} \\ [1.0ex]
   Microsoft \\
   [1.0ex] \small{\textbf{Correspondence:} victorye.dong@gmail.com,  kuanyunlee@microsoft.com}
 }

\begin{document}
\maketitle

\begingroup
\renewcommand{\thefootnote}{\fnsymbol{footnote}}
\footnotetext[1]{Equal contribution.}
\footnotetext[2]{Work done while at Microsoft.}
\endgroup

\begin{abstract}
Data selection for language models is often framed as balancing example quality and diversity. We argue that both are consequences of a more fundamental principle: selected examples should preserve the downstream query space induced by instructions, task signals, or retrieval needs. We present \emph{Greedy Information Projection} (\textsc{GIP}), a query-aligned method that requires only candidate embeddings and a score signal from LLM judgments, metadata, or intrinsic geometry. \textsc{GIP} greedily selects examples whose embedding span explains the largest residual component of task/query scores, yielding a fast matching-pursuit selector. A Gaussian projection view connects this update to maximizing mutual information, equivalently minimizing the residual volume of the query subspace left unexplained by selected data; this explains how quality and diversity emerge from one objective. Empirically, \textsc{GIP} matches or surpasses full-data fine-tuning using small instruction and reasoning subsets, while pretraining and RAG passage-selection experiments show that the same residual principle transfers beyond supervised fine-tuning.
\end{abstract}

\section{Introduction}
\label{sec:intro}

A central question in training large language models (LLMs) is how to select data. Modern LLMs train on large corpora whose size scales with model capacity \citep{kaplan2020scalinglawsneurallanguage,hoffmann2022trainingcomputeoptimallargelanguage}, but once enough data is available, \emph{data quality} becomes the binding constraint on further improvement \citep{achiam2023gpt,llamaTeam2024herd,yang2025qwen3}. A fundamental challenge in pretraining and instruction tuning is therefore to select a subset that maximizes downstream performance under resource constraints.

Existing approaches for selecting fine-tuning data fall into two broad families: \emph{quality-based} filtering---e.g., LLM-judged scoring~\citep{chen2023alpagasus}, perplexity thresholds, or length heuristics~\citep{zhao2024long}---and \emph{diversity-based} construction---e.g., clustering, deduplication~\citep{leeDeduplicatingTrainingData2022}, or coverage objectives~\citep{ge2024clustering,bukharinDataDiversityMatters2024,fan2025combattingdimensionalcollapsellm}. Gradient-influence methods such as LESS~\citep{Xia2024LESS} and importance-resampling methods such as DSIR~\citep{Xie2023DSIR} provide additional axes but typically require target/validation data, model gradients, or both. These approaches give little theoretical guidance on \emph{why} a particular example is valuable and do not unify quality and diversity under a single objective.

\paragraph{This work.}
We propose \emph{Greedy Information Projection} (\textsc{GIP}), a query-aligned framework for selecting examples that preserve downstream task directions. We use ``query'' broadly: it can refer to user instructions, evaluation prompts, reasoning questions, retrieval information needs, or user-intent signals. \textsc{GIP} consumes only (i) feature embeddings $F$ of candidate items and (ii) a score matrix $G$ whose columns are task/query score vectors, such as LLM-judged quality, self-consistency, or simple geometric statistics of $F$. It selects a subset $S$ whose span explains the remaining score signal through a scalable residual matching-pursuit implementation. We make four contributions:

\begin{itemize}
\item \textbf{Query-aligned data selection.} We formulate language-model data selection as preserving a downstream query space induced by instructions, evaluation prompts, reasoning questions, or retrieval needs.
\item \textbf{Greedy Information Projection.} We give a fast residual matching-pursuit algorithm (Algorithm~\ref{alg:gip_mp}) and flexible practical scoring instantiations.
\item \textbf{Projection view of quality and diversity.} We show that relevance and diversity arise from the same residual-projection objective: high-quality examples align with important query directions, while diverse examples cover directions not yet explained by the selected span.
\item \textbf{Empirical validation across NLP settings.} Across instruction following, mathematical reasoning, pretraining, and RAG, \textsc{GIP} shows promising sample-efficient performance with the same residual rule. With only 1--2\% of Alpaca, \textsc{GIP} matches or exceeds the standard Full baseline in 7 of 9 MT-Bench/BBH/Alpaca-Eval-V2 model-metric cells; on GSM8K, it is competitive with gradient-, density-ratio-, and diversity-based baselines and exceeds the Qwen3-8B and Qwen3-32B Full baselines at 10--20\% data (\cref{sec:experiments}).
\end{itemize}

\begin{figure}[H]
\centering
\includegraphics[width=\columnwidth]{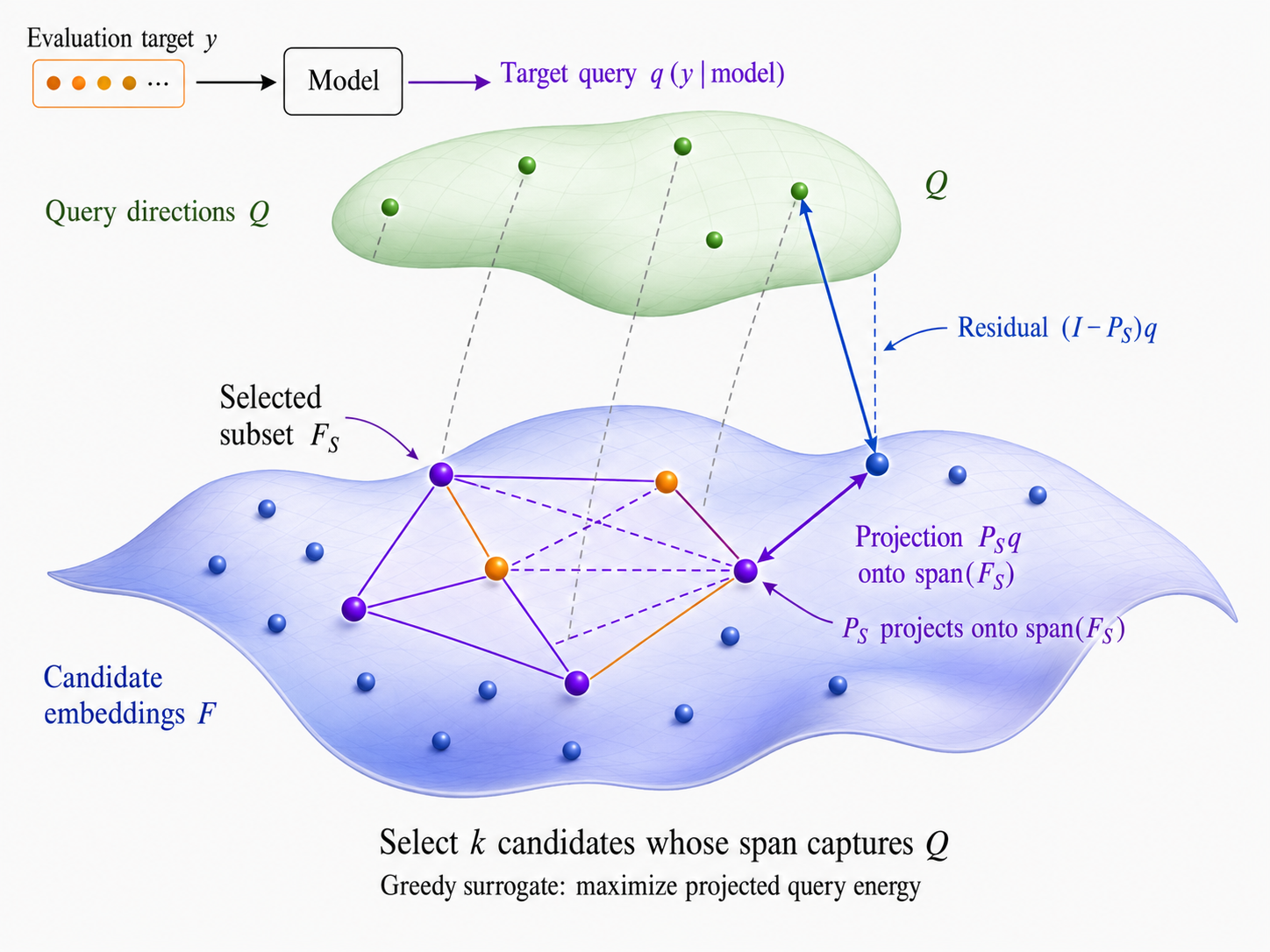}
\caption{Geometric view of \textsc{GIP}. The selected embeddings $F_S$ define $\mathrm{span}(F_S)$ and the projector $P_S$; greedy selection increases the query energy captured by $P_SQ$ while leaving smaller residual directions $(I-P_S)Q$.}
\label{fig:gip_geometric}
\end{figure}

\section{Related Work}
\label{sec:related}

\paragraph{Data selection for LLMs.}
Diversity and quality are both recognized as central to LLM training~\citep{senerActiveLearningConvolutional2018,chenDiversitySyntheticData2024,chenDataJuicerOneStopData2023}. \citet{leeDeduplicatingTrainingData2022} highlight the role of diversity through near-duplicate removal. \citet{bukharinDataDiversityMatters2024} balance quality and diversity in instruction-tuning data, and \citet{duMoDSModelorientedData2023} combine quality, coverage, and necessity. \citet{chen2023alpagasus} use LLM-based quality scores for filtering, and \citet{ge2024clustering} cluster and rank instruction data. \citet{Xia2024LESS} use gradient similarity (LESS); \citet{Xie2023DSIR} use importance resampling (DSIR); \citet{fan2025combattingdimensionalcollapsellm} use submodular file selection to preserve diversity (DiSF). Recent work also selects data by maximizing semantic-space information gain (MIG)~\citep{chen2025mig} or by choosing responses that fit the target model distribution (GRAPE)~\citep{zhang2025grape}. Adjacent training-time approaches improve generation behavior through synthetic instruction generation, feedback-aware self-training, and ranking-based fine-tuning~\citep{wang2023selfinstructaligninglanguagemodels,chai2022fast,liTunaInstructionTuning2023}. Preference and reinforcement-learning methods instead optimize language-model behavior from human, model, or offline reward signals~\citep{ouyang2022training,rafailov2023direct,chen2023semioffline}. \citet{deb2025fishersft} use Fisher-information-based information gain for supervised fine-tuning subset selection. Existing methods generally treat quality and diversity as separate stages or as penalty terms; \textsc{GIP} unifies them in a single residual-matching procedure derivable from one information-theoretic objective.

\paragraph{Retrieval and diversified ranking.}
Quality--diversity tradeoffs are also central in information retrieval based recommendation systems, where a ranked list should contain individually relevant items while covering non-redundant aspects of the query. Recent systems add query-aware diversity or diverse query rewriting for stronger context construction~\citep{khan2026dfrag,li2024dmqrrag,golobokov2022deepgen}. 

\paragraph{Information-theoretic and projection-based selection.}
Mutual information (MI) has a long history in selection and clustering~\citep{knops2006normalized,he2008k,kraskov2005hierarchical,sugiyama2014information,calandriello2014semi,do2021clustering}. Computing MI is generally hard~\citep{kraskov2004estimating}; we exploit Gaussianity to obtain a closed-form log-determinant criterion. Connections to singular-value-based analysis appear in~\citep{lee2021minimax,lee2020linear,lee2022new}. Determinantal point processes (DPPs)~\citep{kulesza2012determinantal,chen2018fast} are the closest kernel-diversity baseline; their score-free entropy objective is the diversity analogue of our score-aware projection objective.

\paragraph{Active learning and coresets.}
\textsc{GIP} is related to active learning~\citep{Settles2009ActiveLearningSurvey} and coreset construction~\citep{bachem2017practical}. Matching pursuit~\citep{mallat1993matching,pati1993orthogonal,tropp2004greed,tropp2007signal} provides the algorithmic backbone of our residual-greedy procedure.

\section{Information-Projection Formalization}
\label{sec:theory}

\subsection{Setup}
\label{sec:setup}

Let $\mathcal{F}=\{f_1,\ldots,f_m\}\subset\mathbb{R}^d$ be unit-norm feature embeddings for $m$ candidate examples, stacked as $F\in\mathbb{R}^{d\times m}$ with $F=[f_1,\ldots,f_m]$. Each candidate is associated with $n$ scalar scores, stacked into a score matrix $G\in\mathbb{R}^{m\times n}$ whose columns $g_1,\ldots,g_n\in\mathbb{R}^m$ are different ``quality'' signals (e.g., LLM judgments, self-consistency, or internal alignment). The selection task is to pick a subset $S\subseteq[m]$ of size at most $k$ that is simultaneously \emph{informative} (aligned with the score signals) and \emph{diverse} (spanning different directions in feature space).

Given the score matrix $G\in\mathbb{R}^{m\times n}$, introduce a latent query embedding $Q\in\mathbb{R}^{d\times n}$ satisfying
\begin{equation}
F^\top Q\approx G.
\label{eq:F_Q_G}
\end{equation}
This $Q$ is only a theoretical device: Algorithm~\ref{alg:gip_mp} runs directly on $F$ and $G$. It lets us interpret score dimensions as query directions in the same geometry as the candidate embeddings.

\subsection{A mutual-information objective}
\label{sec:mi_obj}

Consider an auxiliary standard Gaussian probe $Z\sim\mathcal{N}(0,I_d)$ and the projections $Z_Q:=Q^\top Z\in\mathbb{R}^n$, $Z_{F_S}:=F_S^\top Z\in\mathbb{R}^{|S|}$. The pair $(Z_Q,Z_{F_S})$ is jointly Gaussian with block covariance
\begin{equation}
\Sigma=\begin{bmatrix}Q^\top Q & Q^\top F_S\\F_S^\top Q & F_S^\top F_S\end{bmatrix},
\label{eq:joint_cov}
\end{equation}
and, when the selected covariance is nonsingular, the mutual information has the closed form
\begin{equation}
I(Z_Q;Z_{F_S})=\tfrac12\log\frac{\det(Q^\top Q)\det(F_S^\top F_S)}{\det\Sigma}.
\label{eq:mi_def}
\end{equation}
GIP originates from the intuition to select $S$ by maximizing mutual information; its greedy update is an approximation algorithm for \eqref{eq:mi_def}.

To avoid invertibility assumptions on the selected examples, the projection view can equivalently be written using any orthonormal basis $U_S$ of $\mathrm{span}(F_S)$, with $P_S=U_SU_S^\top$. Then the conditional covariance of $Z_Q$ given $U_S^\top Z$ is $Q^\top(I-P_S)Q$, so maximizing Gaussian mutual information is equivalent, up to ridge or pseudo-determinant regularization in degenerate cases, to minimizing the residual query volume $\det(Q^\top(I-P_S)Q)$. The determinant objective is the ideal query-space preservation criterion; the \textsc{GIP} greedy update optimizes the trace/matching-pursuit surrogate rather than the exact log-determinant.

\begin{remark}[Gaussianity as a modeling device]
Neither data nor query embeddings are assumed Gaussian. The probe $Z$ induces Gaussian random projections whose second-order structure is fully determined by the embedding geometry; mutual information admits a closed-form log-determinant expression, is invariant under orthogonal transformations, and captures total linear dependence between the subspaces spanned by $Q$ and $F_S$.
\end{remark}

\subsection{Geometric projection}
\label{sec:geom}

Maximizing mutual information of the form \eqref{eq:mi_def} is equivalent to optimizing
\begin{align}
\arg\max_S \eqref{eq:mi_def}
= \arg\min_S \det\bigl(Q^\top(I-P_S)Q\bigr).
\label{eq:proj_obj}
\end{align}

The matrix $I-P_S$ projects onto the orthogonal complement of $\mathrm{span}(F_S)$. Minimizing $\det(Q^\top(I-P_S)Q)$ thus selects $F_S$ so that its column span captures as much of $Q$'s column span as possible, simultaneously favoring directions where $Q$ has high energy (quality) and directions that open up new subspace volume (diversity). \Cref{fig:gip_geometric} illustrates the picture.

When scores are absent, $Q$ vanishes from the objective and the MI surrogate reduces to maximizing the entropy $h(Z_{F_S})=\tfrac12\log\det(F_S^\top F_S)$, which is the determinantal point process diversity objective~\citep{kulesza2012determinantal}. \textsc{GIP} thus generalizes DPP-style diversity to score-aware selection through a single formula.

\section{Greedy Information Projection: Empirical Method in Practice}
\label{sec:method}

\begin{algorithm}[t]
\caption{\textsc{GIP} greedy selection for data selection}\label{alg:gip_mp}
\begin{algorithmic}[1]
\REQUIRE Data $F\in\mathbb{R}^{d\times m}$, scores $G\in\mathbb{R}^{m\times n}$, budget $k$, residual scale $\lambda\in[0,1]$ (default $1$)
\ENSURE Selected indices $S$
\STATE $S\leftarrow\emptyset$; $W^{(0)}\leftarrow G^\top$
\STATE Precompute $\Phi\leftarrow F^\top F$
\FOR{$t=1$ \TO $k$}
    \STATE $s_t\leftarrow\arg\max_{j\notin S}\sum_{i=1}^n(W^{(t-1)}_{i,j})^2$
    \STATE $S\leftarrow S\cup\{s_t\}$
    \STATE For all $i\in[n]$, $j\notin S$: $W^{(t)}_{i,j}\leftarrow W^{(t-1)}_{i,j}-\lambda\Phi_{j,s_t}W^{(t-1)}_{i,s_t}$
\ENDFOR
\RETURN $S$
\end{algorithmic}
\end{algorithm}

The algorithm involves a greedy selection where each step selects the candidate with largest residual score energy, after which a residual update subtracts the residual from remaining candidates. Please refer to Algorithm \ref{alg:gip_mp} on how selections and residual updates are performed. In the remainder of the section, we will describe two relaxation steps we take to derive our greedy selection logic.

\paragraph{Step 1: AM--GM relaxation to a trace.}
\begin{equation*}
\det\bigl(Q^\top(I-P_S)Q\bigr)\leq\left(\frac{\tr\bigl(Q^\top(I-P_S)Q\bigr)}{n}\right)^n.
\end{equation*}
The trace is linear in $P_S$ and therefore amenable to greedy maximization. We adopt it as a surrogate:
\begin{align}
\tr\bigl(Q^\top(I-P_S)Q\bigr)
&=\sum_{i=1}^n \left(\|q_i\|^2-\|P_S q_i\|^2 \right).
\label{eq:trace_surrogate}
\end{align}
Note that minimizing the trace depends on $Q$ only through $q_i^\top F=g_i^\top$, so the optimization is computable directly from $G$ without forming $Q$.

\paragraph{Step 2: Greedy matching pursuit on the dual.}
Minimizing \eqref{eq:trace_surrogate} is equivalent to the multi-target sparse approximation problem
\begin{equation}
\min_S\sum_{i=1}^n\min_{\lambda\in\mathbb{R}^{|S|}}\bigl\|q_i-\textstyle\sum_{j\in S}\lambda_j f_j\bigr\|_2^2,
\label{eq:dual_ls}
\end{equation}
because $\|P_S q_i\|^2$ is exactly the energy of $q_i$ explained by the span of $F_S$. We solve \eqref{eq:dual_ls} greedily. At step $t{+}1$, with residual $r_i^{(t)}=q_i-\sum_{j=1}^t\lambda_j^* f_{s_j}$, the optimal single-atom step satisfies $\min_\lambda\|r_i-\lambda f_s\|_2^2=\|r_i\|^2-(r_i^\top f_s)^2$ since $\|f_s\|=1$, and the selection rule reduces to
\begin{equation}
s_{t+1}=\arg\max_{s\notin S_t}\sum_{i=1}^n(r_i^\top f_s)^2.
\label{eq:omp_sel}
\end{equation}
Defining $W_{i,j}^{(t)}=r_i^{(t)\top}f_j$, the above suggests a residual update $r_i^{(t+1)}=r_i^{(t)}-(r_i^{(t)\top}f_{s_{t+1}})f_{s_{t+1}}$. Note that because $W^{(0)}_{i,j}=q_i^\top f_j=G_{j,i}$, the algorithm requires only $F$ and $G$ (not $Q$).



\subsection{\textsc{GIP} greedy selection}
\label{sec:greedy_mp}

Based on the above two steps, \textsc{GIP} maintains a \emph{residual score matrix} $W^{(t)}\in\mathbb{R}^{n\times m}$ whose $i$-th row represents how much of the $i$-th score signal has not yet been explained by previously selected items, together with a precomputed inner-product matrix $\Phi=F^\top F\in\mathbb{R}^{m\times m}$. Initially, $W^{(0)}=G^\top$. At step $t$, we select the candidate whose feature explains the most variance jointly across all $n$ residual scores, then subtract the rank-one component aligned with $f_{s_t}$ so that subsequent selections are biased away from items in the same direction:
\begin{align}
&s_t = \arg\max_{j\in[m]\setminus S_{t-1}} \sum_{i=1}^n \bigl(W^{(t-1)}_{i,j}\bigr)^2, \label{eq:mp_select}\\
&W^{(t)}_{i,j} \leftarrow W^{(t-1)}_{i,j} - \lambda\,\Phi_{j,s_t}\cdot W^{(t-1)}_{i,s_t}. \label{eq:mp_update}
\end{align}
The residual scale $\lambda\in[0,1]$ controls how aggressively selected directions are removed. The default $\lambda=1$ gives the full matching-pursuit residual used in our fine-tuning experiments. The endpoint $\lambda=0$ masks only the selected item and therefore behaves like greedy top-score selection under the current score matrix, while intermediate values retain more direct query relevance and apply a softer diversity correction. Thus \textsc{GIP}$(\lambda)$ is a single-parameter balance between score preservation and residual coverage. We repeat for $k$ steps and return $S=S_k$; Algorithm~\ref{alg:gip_mp} summarizes the procedure.

\subsection{Computational complexity}
\label{sec:complexity}

Since each residual update only requires similarities $\Phi_{j,s_t}=f_j^\top f_{s_t}$ between candidates and the newly selected item, one \textit{cached} approach is to precompute a full Gram matrix $\Phi=F^\top F$ at cost $O(m^2d)$ and memory $O(m^2)$, which can be amortized across budgets or score variants; after that, each greedy step costs $O(mn)$ for selection and $O(mn)$ for the residual update, giving $O(mnk)$ selection time. Alternatively, for large candidate pools, the vector $\Phi_{:,s_t}$ can be computed on the fly (\textit{streamed}) in $O(md)$ time per step, avoiding $O(m^2)$ memory at the cost of $O(mdk)$ additional computation. In our main settings $n$ is small, so the cached-kernel selector is effectively linear in $m$ and $k$ after preprocessing.

\Cref{tab:complexity-summary} compares against common baselines. \textsc{GIP} uses candidate embeddings together with a score matrix, but does not require held-out validation data or model-gradient computation. Unlike greedy DPP MAP, which scales as $O(mk^2)$ in $k$, cached-kernel \textsc{GIP} scales as $O(mk)$ in $k$ when the score dimension $n$ is small. Wall-clock and peak-RAM measurements (\cref{tab:runtime_memory}) show that the preprocessing cost is amortizable in our evaluated settings.

\begin{table}[b]
\centering\footnotesize
\renewcommand{\arraystretch}{1.0}
\setlength{\tabcolsep}{1.0pt}
\caption{Complexity and selector requirements. C/S denotes cached/streamed and $m_v=m_\text{val}$. Val. marks whether a method requires held-out target or validation data, and Grad. marks model-gradient computation. $\tilde{O}$ hides polylogarithmic factors, and $l$ counts tokens per DSIR candidate.}
\label{tab:complexity-summary}
\begin{tabular*}{\columnwidth}{@{\extracolsep{\fill}}lcccc@{}}
\toprule
\textbf{Method} & \textbf{Prep.} & \textbf{Selection} & \textbf{Val.} & \textbf{Grad.} \\
\midrule
\textsc{GIP} (C) & $O(m^2d)$ & $O(mk)$ & No & No \\
\textsc{GIP} (S) & $O(m)$ & $O(mkd)$ & No & No \\
DPP (C) & $O(m^2d)$ & $O(mk^2)$ & No & No \\
DPP (S) & $O(m)$ & $O(mk^2d)$ & No & No \\
LESS & $O(mN)$ & $O(mm_vd)$ & Yes & Yes \\
DiSF & $O(m)$ & $\tilde{O}(mkd^2)$ & Yes & No \\
DSIR & $O((m+m_v)(d+l))$ & $O((m+m_v)d)$ & Yes & No \\
\bottomrule
\end{tabular*}
\end{table}

\section{Experiments}
\label{sec:experiments}

We evaluate \textsc{GIP} as a query-aligned selection rule in four NLP regimes. First, instruction-following selection tests whether a small Alpaca subset can preserve broad conversational supervision. Second, mathematical-reasoning selection tests whether the same method can choose target-domain GSM8K examples under exact-answer evaluation. Third, pretraining selection tests whether the residual rule can choose useful web-text chunks before supervised training. Fourth, query-time RAG selection tests whether the residual update transfers without fine-tuning. This organization separates the application surfaces while keeping the underlying object fixed: select candidates whose span explains the task query signals.

\subsection{Instruction following and mathematical reasoning}
\label{sec:main_results}

We first evaluate \textsc{GIP} in supervised fine-tuning for instruction following and mathematical reasoning. The Alpaca experiments measure broad instruction-following transfer from small general-purpose subsets, while the GSM8K experiments measure whether selected reasoning examples improve exact-answer mathematical performance across data budgets and model scales. For both tasks, candidates are represented by normalized ModernBERT-Base embeddings~\citep{warner2024smarter}; Alpaca examples concatenate \emph{[instruction; context; response]}, and GSM8K uses the question as the instruction.

\paragraph{Task setup.}
For instruction tuning, we select 512 and 1{,}000 examples from Alpaca-52k~\citep{taori2023alpaca} and evaluate Mistral-7B-v0.1~\citep{jiang2023mistral7b}, Qwen3-4B, and Qwen3-8B~\citep{yang2025qwen3} on MT-Bench~\citep{zheng2023judging}, BBH~\citep{Suzgun2022BBH}, and Alpaca-Eval-V2~\citep{dubois2024length}. For mathematical reasoning, we select 2.5--20\% of GSM8K~\citep{cobbe2021training} and evaluate exact-match accuracy on Mistral-7B, Qwen3-8B, and Qwen3-32B. Fine-tuning baselines include Random, Full, Alpaca-longest~\citep{zhao2024long}, Alpagasus~\citep{chen2023alpagasus}, CaR~\citep{ge2024clustering}, DSIR~\citep{Xie2023DSIR}, DiSF~\citep{fan2025combattingdimensionalcollapsellm}, and LESS~\citep{Xia2024LESS}; DSIR receives the evaluation set as its target domain, giving it an oracle advantage on GSM8K.

\paragraph{Fine-tuning evaluation protocol.}
All selected subsets and baselines are fine-tuned from the same base checkpoint under the shared LoRA recipe in \cref{tab:model_train_details}. MT-Bench responses are generated with vLLM~\citep{kwon2023efficient} and judged by GPT-5 on the standard 1--10 MT-Bench scale; Alpaca-Eval-V2 reports length-controlled win rate; BBH and GSM8K use exact-match accuracy with the benchmark's standard answer normalization. For Qwen3 models, thinking mode is disabled during both fine-tuning and evaluation.

\paragraph{Baseline implementation.}
For Alpaca, we evaluate the released or paper-specified selection rules at matched 1k-example budgets where applicable (Alpaca-longest, Alpagasus, CaR) and retrain all reported baselines with the same fine-tuning and evaluation pipeline. For GSM8K, Random, DSIR, DiSF, LESS, and \textsc{GIP} variants use the same percentage budgets. DSIR and LESS use the same target/validation split for distribution matching or gradient influence, so they are strong target-aware baselines, while \textsc{GIP} uses only candidate embeddings plus its score signal.

\paragraph{Instruction-tuning scoring schemes.}
\textsc{GIP} accepts any score matrix $G$. For fine-tuning selection, we study two practical instantiations. First is multi-attribute scoring (MA): for every candidate, we concatenate \emph{[instruction; context; response]} and ask GPT-4o to score coherence, accuracy, helpfulness, and difficulty on a 0--5 scale. We cache these four attributes and use their sum as a single score vector ($n=1$); the full prompt is in \cref{tab:evaluation_rubric}. The second variant is training-data-aware self-compression (SC) that uses no external judge and derives the score from the data geometry:
\[
g_i = \sum_{j=1}^m f_i^\top f_j = (F^\top F\,\mathbf{1}_m)_i.
\]
This row-sum score measures representativeness in the candidate pool. Combined with the residual update~\eqref{eq:mp_update}, GIP+SC trades representativeness against coverage.

\input{generated/main_training_tables}

\paragraph{Sample efficiency.}
\Cref{tab:mt-bench-bbh-results} shows the strongest data-efficiency result: with only 1--2\% of Alpaca, GIP+MA/GIP+SC exceed the Full-data baseline in 7 of 9 model-metric cells across MT-Bench, BBH, and Alpaca-Eval-V2. \Cref{tab:GSM8K_merged} shows more mixed but still favorable reasoning results: GIP+MA exceeds the Qwen3-8B and Qwen3-32B Full baselines at 10--20\% data. Under our fixed fine-tuning protocol, \textsc{GIP}-selected subsets match or exceed the standard full-data baseline, especially when the source pool contains noisy or redundant examples. Because Alpaca is noisy and redundant, 1\% can occasionally outperform 2\% after the highest-marginal-value examples have already been selected.

\paragraph{Robustness across models.}
Gains are stable across Mistral and Qwen3 at 4B--32B parameters and across benchmarks probing distinct capabilities. The strongest non-\textsc{GIP} baseline on GSM8K is DSIR, but its edge is concentrated at the 5\% budget on the Qwen3 models; at the other budgets, GIP+MA matches or exceeds DSIR across every base model. On Mistral-7B, GIP+MA consistently leads DSIR. Note that DSIR enjoys oracle access to the evaluation distribution via its target-domain mechanism, whereas \textsc{GIP} does not.

\paragraph{Complementarity of GIP+MA and GIP+SC.}
GIP+MA performs best on instruction-following and preference metrics by exploiting external multi-attribute judgments, while GIP+SC---using only intrinsic geometry---is competitive on reasoning benchmarks and excels on GSM8K with Mistral-7B at 10--20\%. This pattern supports our view that quality and diversity are jointly handled by the same residual mechanism, with the score channel determining which signal is emphasized.

\paragraph{Comparison with Information-based methods} To compare with information-based method FisherSFT \citep{deb2025fishersft}, we evaluate it using the authors' released implementation at the 2.5\% and 20\% data budgets on Mistral-7B and Qwen3-8B. Across these four matched settings, both GIP variants outperform FisherSFT. Detailed results are reported in \cref{tab:fishersft-comparison}.

\paragraph{Ablations.}
\label{sec:ablation}
We include three targeted ablations for training-time selection. First, to test whether \textsc{GIP} replaces or complements source-pool cleaning, we compare GIP+SC at 1\% data against Full on Mistral-7B across MT-Bench categories for cleaned and non-cleaned Alpaca (\cref{fig:mtbench_mistral_per_category}). Cleaning improves both methods, especially on open-ended categories, and narrows the gap enough that 1\% GIP+SC matches cleaned Full overall. Thus GIP+SC benefits from source-pool quality rather than replacing it.

Second, to measure sensitivity to the feature geometry, we swap the encoder feeding GIP+SC on GSM8K between ModernBERT and a Qwen-reasoning embedding. Qwen-reasoning gives small gains on most rows, particularly for Qwen3-32B at 10--20\%; ModernBERT remains competitive, especially on Mistral-7B at 10\% (\cref{tab:embedding_ablation}).

\begin{table}[H]
\caption{Embedding ablation for GIP+SC on GSM8K.}
\label{tab:embedding_ablation}
\centering\footnotesize
\begin{tabular}{l|c|cc}
\toprule
Model & Data & ModernBERT & Qwen-reasoning \\
\midrule
Mistral-7B & 20\% & 49.81 & \textbf{50.27} \\
Mistral-7B & 10\% & \textbf{46.63} & 46.25 \\
Qwen3-32B & 20\% & 87.57 & \textbf{88.02} \\
Qwen3-32B & 10\% & 87.34 & \textbf{87.72} \\
\bottomrule
\end{tabular}
\end{table}

Finally, to test selection stability under embedding perturbations, we add Gaussian noise to GSM8K embeddings and rerun GIP+SC. Selected-subset IoU remains $\geq 85\%$ for $\sigma\leq 10^{-3}$ and above $60\%$ at $\sigma{=}10^{-2}$, so mild embedding drift has little effect and severe noise still preserves most core items (\cref{tab:stability}).

\begin{figure*}[t!]
\centering
\begin{subfigure}[t]{0.51\textwidth}
\vspace{0pt}
\centering
\includegraphics[width=\linewidth]{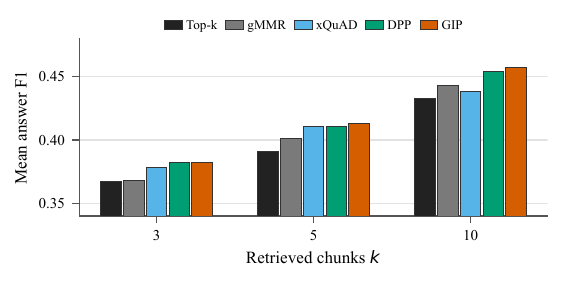}
\caption{RAG oracle mean answer F1 by context budget.}
\end{subfigure}
\hfill
\begin{subfigure}[t]{0.47\textwidth}
\vspace{0pt}
\centering
\includegraphics[width=\linewidth]{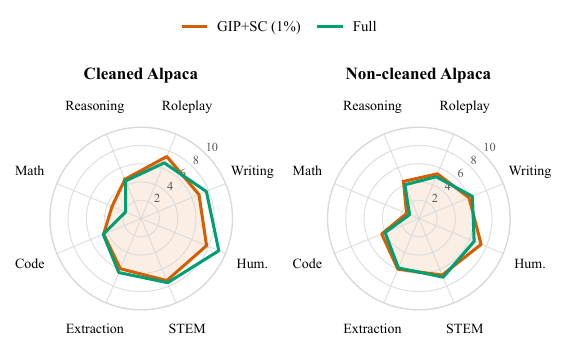}
\caption{MT-Bench per-category scores on Mistral-7B.}
\label{fig:mtbench_mistral_per_category}
\end{subfigure}
\vspace{-0.5em}
\caption{RAG and MT-Bench ablations. Left: oracle best-over-sweep mean answer F1 grouped by retrieved-context budget; tunable methods are swept over their tradeoff parameters. Right: GIP+SC at 1\% data versus Full on cleaned and non-cleaned Alpaca, judged by GPT-5.}
\label{fig:rag-answer-f1}
\end{figure*}

\subsection{Pretraining data selection}
\label{sec:pretraining-selection}

We test query-aligned selection for language-model pretraining using DCLM-Baseline~\citep{li2024datacomp}. The initial comparison selects from a 1.69M-chunk pool at 0.8B and 1.6B tokens and trains a 411M-parameter OpenLM model~\citep{open_lm}. For the larger follow-up, GIP and DSIR each select 2M chunks (4.1B tokens) from an enlarged DCLM-Baseline pool and train the OpenLM-1B configuration for one epoch. The two follow-up runs use the same seed and optimization recipe.

\begin{table*}[t]
\centering\footnotesize
\caption{Per-task results from single 411M-parameter OpenLM pretraining runs at 0.8B and 1.6B tokens. Best results are in \textbf{bold}; second-best results are \underline{underlined}. The OpenLM-1B follow-up and scale comparison are shown in \cref{fig:pretrain-scale-comparison}; three-seed averages are reported in \cref{tab:pretrain-detailed}.}
\label{tab:pretrain-main}
\setlength{\tabcolsep}{2.0pt}
\begin{tabular*}{\textwidth}{@{\extracolsep{\fill}}llcccccccc}
\toprule
\textbf{Tokens} & \textbf{Selector} & \textbf{HS} & \textbf{PIQA} & \textbf{ARC-E} & \textbf{ARC-C} & \textbf{WG} & \textbf{OBQA} & \textbf{SciQ} & \textbf{Mean} \\
\midrule
0.8B & Random & 0.298 & \secondbest{0.544} & 0.300 & 0.198 & 0.489 & 0.250 & 0.335 & 0.345 \\
0.8B & DSIR & 0.290 & \best{0.545} & 0.296 & \best{0.218} & \secondbest{0.494} & \best{0.262} & \secondbest{0.352} & 0.351 \\
0.8B & DiSF & \secondbest{0.299} & 0.543 & \secondbest{0.315} & \best{0.218} & 0.491 & \best{0.262} & 0.340 & \secondbest{0.353} \\
0.8B & \textsc{GIP} & \best{0.306} & 0.540 & \best{0.317} & \secondbest{0.217} & \best{0.496} & \secondbest{0.260} & \best{0.390} & \best{0.361} \\
\midrule
1.6B & Random & 0.317 & 0.572 & 0.332 & 0.205 & 0.498 & 0.236 & 0.475 & 0.376 \\
1.6B & DSIR & \secondbest{0.326} & \best{0.590} & 0.331 & 0.211 & \best{0.504} & \secondbest{0.242} & \secondbest{0.513} & \secondbest{0.388} \\
1.6B & DiSF & 0.320 & \secondbest{0.585} & \secondbest{0.339} & \best{0.215} & \secondbest{0.499} & 0.240 & 0.500 & 0.385 \\
1.6B & \textsc{GIP} & \best{0.334} & 0.574 & \best{0.350} & \secondbest{0.214} & 0.480 & \best{0.262} & \best{0.517} & \best{0.390} \\
\bottomrule
\end{tabular*}
\end{table*}

\begin{figure}[t]
\centering
\includegraphics[width=\columnwidth]{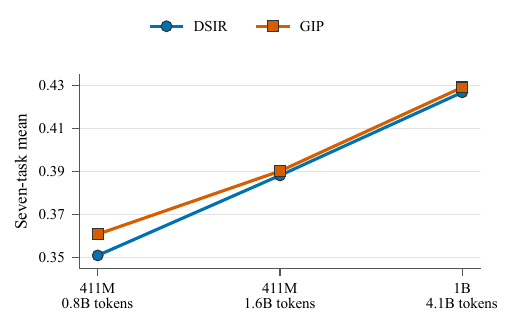}
\caption{Seven-task pretraining mean for DSIR and \textsc{GIP} in the 411M/0.8B-token, 411M/1.6B-token, and OpenLM-1B/4.1B-token settings. The 4.1B-token runs use the same seed and optimization recipe; this setting changes both model size and token budget and therefore measures their combined scale-up.}
\label{fig:pretrain-scale-comparison}
\end{figure}

We compare against random sampling, DSIR importance resampling~\citep{Xie2023DSIR}, and DiSF relevance-diversity selection~\citep{fan2025combattingdimensionalcollapsellm}. All target-aware selectors use the same target data: 10k chunks from the English Wikipedia data corpus \citep{wikimedia2023wikipedia}, tokenized and chunked with the same pipeline as the DCLM pool. This target data is separate from the downstream evaluation tasks; HellaSwag~\citep{zellers2019hellaswag}, PIQA~\citep{bisk2020piqa}, ARC~\citep{clark2018think}, WinoGrande~\citep{sakaguchi2020winogrande}, OpenBookQA~\citep{mihaylov2018openbookqa}, and SciQ~\citep{welbl2017crowdsourcing} are used only after training.

This design fixes the candidate pool and target source while varying the selection objective and feature representation. DSIR consumes the shared target through hashed token n-gram distributions, whereas DiSF and \textsc{GIP} use dense ModernBERT chunk embeddings. For scalable pretraining selection, \textsc{GIP} clusters the target embeddings into $C$ target modes using the $k$-means algorithm, replacing the full set of target directions with a small set of mode centroids before applying the residual update (as in \cref{alg:gip_mp}). For model training, all methods use AdamW~\citep{loshchilov2019decoupled} with learning rate $3\cdot10^{-4}$, sequence length 2048, bfloat16 mixed precision, and a cosine schedule. Evaluation uses lm-evaluation-harness~\citep{eval-harness} with a limit of 1000 examples per task; metrics are normalized accuracy for HellaSwag, ARC, and OpenBookQA, and accuracy for PIQA, WinoGrande, and SciQ. Additional implementation details are in \cref{app:pretraining-pilot}.

\Cref{tab:pretrain-main} reports the 411M per-task results, while \cref{fig:pretrain-scale-comparison} compares DSIR and \textsc{GIP} across all three settings. \textsc{GIP} has the higher seven-task mean in each setting; in the OpenLM-1B follow-up, its mean is 0.4290 compared with 0.4267 for DSIR. Because the 4.1B-token setting changes both model size and token budget, it tests the combined larger-scale setting rather than isolating either factor. At 1.6B tokens, the three-seed averages in \cref{tab:pretrain-detailed} remain above Random and DSIR; the per-task averages match or exceed DSIR on all seven tasks.

\subsection{Retrieval-Augmented Question Answering}
\label{sec:retrieval-style-experiments}

We next evaluate the same residual mechanism when selection happens at inference time rather than during training. The primary setting is RAG passage selection, where each question induces a score vector over candidate document chunks. Our reported \textsc{GIP} method is the query-only residual-scaled \textsc{GIP}$(\lambda)$ variant from \eqref{eq:mp_update}; baselines are vanilla top-k, gMMR, xQuAD, and DPP$(\alpha)$ \citep{kulesza2012determinantal,chen2018fast}.

\paragraph{Experiment setting.}
For retrieval-augmented QA, we use LongBench HotpotQA \citep{yang2018hotpotqa}, 2WikiMQA \citep{ho2020constructing}, MuSiQue \citep{trivedi2022musique}, and MultiFieldQA-EN \citep{bai2023longbench}. We evaluate retrieved-context budgets $k\in\{3,5,10\}$, generate answers with Qwen3-8B \citep{yang2025qwen3} with thinking disabled, and evaluate them using the official LongBench answer-F1 metric. Exact sample counts, passage construction, embedding models, and parameter sweeps are reported in \cref{app:rag-setup}.

\Cref{fig:rag-answer-f1} summarizes the macro answer-F1 frontier across budgets, with per-benchmark values in \cref{tab:rag-budget-oracle}. \textsc{GIP} is strongest at $k=5$ and $k=10$ and is effectively tied with DPP at $k=3$. DPP remains the strongest competitor.

\paragraph{Impact of query depth.}
To test whether \textsc{GIP} benefits from a query space beyond a single question vector, we augment each question with planner-generated decompositions. Thus $d_Q=1$ uses only the original question, while $d_Q>1$ appends planner subquestions as additional score columns. These extra columns ask \textsc{GIP} to cover multiple decomposed information needs rather than a single aggregate query embedding. \Cref{tab:rag-qdim-ablation} shows that planner dimensions improve the fixed-$\lambda$, $k=5$ ablation, suggesting that decomposition helps \textsc{GIP} cover complementary information needs. Planner and embedding details are in \cref{app:rag-setup}.

\begin{table}[H]
\caption{RAG query-dimensionality ablation at $k=5$ with fixed $\lambda=0.8$. Best in \textbf{bold}, second-best \underline{underlined}.}
\label{tab:rag-qdim-ablation}
\centering\footnotesize
\begin{adjustbox}{width=\columnwidth}
\begin{tabular}{c|ccccc}
\toprule
\textbf{$d_Q$} & \textbf{Hotpot} & \textbf{2Wiki} & \textbf{MultiField} & \textbf{MuSiQue} & \textbf{Mean} \\
\midrule
1 & \secondbest{0.498} & 0.426 & 0.415 & 0.276 & 0.404 \\
3 & 0.486 & \secondbest{0.442} & \best{0.424} & \secondbest{0.283} & \secondbest{0.409} \\
5 & \best{0.511} & \best{0.448} & \secondbest{0.422} & \best{0.285} & \best{0.416} \\
\bottomrule
\end{tabular}
\end{adjustbox}
\end{table}

\section{Conclusion}
\label{sec:conclusion}

We introduced \textsc{GIP}, a query-aligned data-selection method that greedily selects examples whose embeddings explain the remaining task/query score directions. The method uses only candidate embeddings and score signals, can reuse or stream pairwise similarities during matching-pursuit selection, and avoids model-gradient computation. Across instruction tuning, mathematical reasoning, pretraining, and RAG passage selection, the same residual update yields sample-efficient subsets and competitive retrieval frontiers. The projection view connects this update to mutual information and gives a unified explanation: high-quality examples align with important query directions, while diversity comes from covering directions not yet explained.

\section*{Limitations}
\textbf{Embedding dependence.} \textsc{GIP} relies on encoder geometry, so poor embeddings or domain-mismatched encoders can weaken capability coverage. \textbf{Quality signals.} Input score signals can carry bias, so \textsc{GIP} may under-cover examples undervalued by those signals. \textbf{Scope.} We study instruction tuning, reasoning, pretraining, and RAG in English text settings; code, multilingual, multimodal, and much larger pretraining pools remain future work. Formal bias-aware algorithms are not considered and are left as future work.

\bibliography{iclr2026_conference}

\clearpage
\appendix

\section{Experimental Setup and Reproducibility}
\label{app:reproducibility}
\subsection{Datasets and Computational Resources}
\label{app:data-compute}

\begin{table*}[t]
\centering
\footnotesize
\caption{Datasets used in our experiments, their stated licenses, and their roles. We report a dataset license when the source states one; otherwise, we identify the license of the official release repository. Documents incorporated from third-party sources retain their original terms.}
\label{tab:datasets-licenses}
\begin{tabularx}{\textwidth}{@{}p{0.43\textwidth}X@{}}
\toprule
\textbf{Dataset and license} & \textbf{Use in this work} \\
\midrule
\multicolumn{2}{@{}l}{\textit{Fine-tuning and evaluation}} \\
Alpaca-52k~\citep{taori2023alpaca} (CC BY-NC 4.0) & Instruction-tuning candidate pool \\
GSM8K~\citep{cobbe2021training} (release repository: MIT) & Mathematical-reasoning selection, fine-tuning, and held-out evaluation. \\
MT-Bench~\citep{zheng2023judging} (release repository: Apache 2.0) & Multi-turn instruction-following evaluation. \\
BIG-Bench Hard (BBH)~\citep{Suzgun2022BBH} (release repository: MIT) & Exact-match reasoning evaluation. \\
AlpacaEval 2~\citep{dubois2024length} (CC BY-NC 4.0) & Length-controlled instruction-following evaluation. \\
\addlinespace
\multicolumn{2}{@{}l}{\textit{Pretraining and evaluation}} \\
DCLM-Baseline~\citep{li2024datacomp} (CC BY 4.0) & Candidate pool for all pretraining data-selection methods. \\
English Wikipedia, November 2023~\citep{wikimedia2023wikipedia} (CC BY-SA 4.0 and GFDL) & Target corpus used by the target-aware selectors. \\
HellaSwag~\citep{zellers2019hellaswag} (release repository: MIT) & Commonsense completion evaluation. \\
PIQA~\citep{bisk2020piqa} (AFL 3.0) & Physical-commonsense evaluation. \\
ARC-Easy/Challenge~\citep{clark2018think} (CC BY-SA 4.0) & Grade-school science evaluation. \\
WinoGrande~\citep{sakaguchi2020winogrande} (CC BY; version not specified) & Commonsense coreference evaluation. \\
OpenBookQA~\citep{mihaylov2018openbookqa} (release repository: Apache 2.0) & Open-book science evaluation. \\
SciQ~\citep{welbl2017crowdsourcing} (CC BY-NC 3.0) & Science question-answering evaluation. \\
\addlinespace
\multicolumn{2}{@{}l}{\textit{Retrieval-augmented generation evaluation}} \\
HotpotQA~\citep{yang2018hotpotqa} (CC BY-SA 4.0) & Multi-hop retrieval and question answering. \\
2WikiMQA~\citep{ho2020constructing} (release repository: Apache 2.0) & Multi-hop retrieval and question answering. \\
MuSiQue~\citep{trivedi2022musique} (CC BY 4.0) & Compositional multi-hop retrieval and question answering. \\
MultiFieldQA-EN~\citep{bai2023longbench} (LongBench repository: MIT; source terms apply) & Long-context retrieval and question answering. \\
\bottomrule
\end{tabularx}
\end{table*}

\textbf{Computational resources.} Experiments were run on NVIDIA A100 80GB and V100 32GB GPUs, depending on the model and experiment. Experiment-specific GPU counts, model configurations, and evaluation protocols are reported in the following setup subsections.
 
\subsection{Fine-Tuning Setup}
\label{app:training}
Hyperparameters and training details are summarized in \cref{tab:model_train_details}. All models are trained in the Hugging Face framework with LoRA fine-tuning~\citep{hu2021loralowrankadaptationlarge}. We apply LoRA to all linear layers for Mistral and Qwen3 runs with $\{r=8, \alpha=16\}$. For Qwen3 models we disable thinking mode during supervised fine-tuning and downstream evaluation. For fair comparison with baselines, we use the same context length and epoch budget even when these differ from~\citep{chen2023alpagasus}. The Qwen3-32B GSM8K runs use the same Qwen3 LoRA configuration as the other Qwen3 models, but train on GSM8K subsets rather than Alpaca-derived subsets.

We reuse the evaluation prompts from MT-Bench~\citep{zheng2023judging} and lm-evaluation-harness~\citep{eval-harness}; GSM8K and BBH use exact match with standard answer normalization. Fine-tuning tables report mean $\pm$ standard error across evaluation examples where available. We fix the data-loader, selection, CUDA/cuDNN, and training seeds wherever the framework exposes deterministic control.

\begin{table*}[t]
\centering\small
\caption{LoRA fine-tuning hyperparameters. Qwen3-32B uses the same LoRA configuration as the other Qwen3 models.}
\label{tab:model_train_details}
\begin{tabular}{lcccccc}
\toprule
\textbf{Dataset} & \textbf{Size} & \textbf{GPUs} & \textbf{Epochs} & \textbf{LR} & \textbf{Sched.} & \textbf{Ctx.\ Len.} \\
\midrule
Alpaca-52k    & 52k & 8 & 10 & 3e-5 & Linear & 2048 \\
Alpagasus-1k  & 1k  & 4 & 10 & 3e-5 & Linear & 2048 \\
CaR-1k        & 1k  & 4 & 10 & 3e-5 & Linear & 2048 \\
Random-1k     & 1k  & 4 & 10 & 3e-5 & Linear & 2048 \\
GIP-512        & 512 & 4 & 10 & 3e-5 & Linear & 2048 \\
\bottomrule
\end{tabular}
\end{table*}

\subsection{Pretraining Setup}
\label{app:pretraining-pilot}

\paragraph{Candidate pool and target.}
The initial candidate pool contains 1,685,901 DCLM-Baseline chunks, tokenized with the GPT-NeoX tokenizer~\citep{black2022gptneox}. For each selector, we materialize a WebDataset training stream from this pool at the 0.8B- or 1.6B-token budget. For the 4.1B-token follow-up, we enlarge the DCLM-Baseline pool and materialize a method-specific stream of 2M selected chunks for both GIP and DSIR. The exact budget is 4.096B tokens ($2\text{M}\times2{,}048$). The target distribution was intended to mix Wikipedia and books, but both book sources failed to load at runtime. The final target is therefore 10k chunks from the English Wikipedia data corpus, using the November 2023 dump~\citep{wikimedia2023wikipedia}.

\paragraph{Feature Construction.}
We compute two reusable feature matrices. DSIR uses a 10k-dimensional sparse token n-gram view: unigram buckets are token ID modulo 10k, bigram buckets use a multiplicative hash modulo 10k, and vectors are L1-normalized. \textsc{GIP} and DiSF use 768-dimensional ModernBERT dense embeddings~\citep{warner2024smarter}: the first 512 decoded tokens of each chunk are embedded, mean-pooled over unmasked positions, cast from bfloat16 to float32, and L2-normalized. The dense pool feature matrix is 5.18 GB; the sparse n-gram matrix is 67.4 GB.

\paragraph{Selectors.}
Random uniformly samples pool indices up to the requested token budget. DSIR computes n-gram importance weights against the same wiki-only target distribution and samples without replacement using Gumbel top-$k$. DiSF uses a relevance-plus-diversity greedy score $\beta\,\mathrm{rel}_i+(1-\beta)(-\max_{s\in S} f_i^\top f_s)$ with $\beta=0.5$ and omits the original sparse-filtration pre-pruning step.

For the main pretraining row, \textsc{GIP} uses target modes rather than a single target centroid. We run k-means on the target embeddings to obtain centroids $\{c_\ell\}_{\ell=1}^{8}$ and initialize
\[
m_i=N\max_\ell \max(0,f_i^\top c_\ell).
\]
At each step, \textsc{GIP} selects $\arg\max_i m_i$ and applies the residual update
\[
m_i \leftarrow m_i-r\,m_s\,f_i^\top f_s,
\]
with residual scale $r=1$ and no appended bias term. For efficiency, we perform selection in an 8-GPU sharded mode, where each GPU greedily selects its share of the target budget from its shard and selected ids are gathered at the end.

\paragraph{Training and evaluation.}
At 0.8B and 1.6B tokens, all selectors train the same 411M-parameter OpenLM model~\citep{open_lm} on 2 nodes $\times$ 8 A100-80GB GPUs. The model has 24 decoder blocks, hidden size 1024, 8 attention heads, sequence length 2048, and the GPT-NeoX vocabulary. The recipe uses DDP, gradient checkpointing, per-GPU batch size 32, global batch size 512, AdamW~\citep{loshchilov2019decoupled} with learning rate $3\cdot10^{-4}$, weight decay 0.1, cosine schedule, 200 warmup steps, bfloat16 mixed precision, and sequence length 2048. The 0.8B and 1.6B runs use separately materialized selected streams from the same full candidate pool, so the larger budget increases the amount of selected training data rather than replaying a fixed 0.8B-token subset.

At 4.1B tokens, GIP and DSIR train the OpenLM-1B configuration for one epoch on their respective 2M-chunk selected streams. Both runs use the same seed and optimization recipe. Evaluation follows the same seven-task protocol at every compute level: lm-evaluation-harness~\citep{eval-harness} with a limit of 1,000 examples per task on HellaSwag~\citep{zellers2019hellaswag}, PIQA~\citep{bisk2020piqa}, ARC-Easy/Challenge~\citep{clark2018think}, WinoGrande~\citep{sakaguchi2020winogrande}, OpenBookQA~\citep{mihaylov2018openbookqa}, and SciQ~\citep{welbl2017crowdsourcing}. The 1.6B-token comparison averages three training seeds while holding each selected subset fixed. The 0.8B-token runs use one seed, and the GIP and DSIR runs in the 4.1B-token follow-up use the same single seed.

\subsection{RAG Setup}
\label{app:rag-setup}
For retrieval-augmented QA, we evaluate 200 examples from each of HotpotQA, 2WikiMQA, and MuSiQue, and 150 examples from MultiFieldQA-EN. We split contexts into 200-word passages and encode questions, planner subquestions, and passages as normalized vectors with the Sentence-Transformers Multi-QA MPNet base cosine encoder, a semantic-search model trained on question--answer pairs \citep{sentenceTransformersMultiQaMpnet,reimers2019sentence,song2020mpnet}. We score generated answers with the official LongBench answer-F1 metric. For the main RAG comparison, \textsc{GIP} uses the original question embedding as $Q$. We sweep \textsc{GIP}$(\lambda)$ over $\{0,0.3,0.6,0.8,1.0\}$ and DPP$(\alpha)$ over $\{0.5,1,2,4,8\}$; fixed baselines are vanilla top-k, gMMR~\citep{carbonell1998use}, and xQuAD~\citep{santos2010exploiting}.

For the query-depth ablation, we use Qwen2.5-7B-Instruct~\citep{qwen2025qwen25technicalreport} as a deterministic planner to generate 2--5 atomic subquestions. We embed them with the same retriever encoder and append up to $d_Q-1$ subquestions to the original question embedding; if fewer subquestions are available, we use all of them. The fixed-setting ablation uses $k=5$ and $\lambda=0.8$.

\section{Multi-Attribute Scoring Prompt}
\label{app:rubric}
The prompt used for the GIP+MA scoring is shown in \cref{tab:evaluation_rubric}. The rubric contains four sections (Coherence, Accuracy, Helpfulness, Difficulty), each rated on a 0--5 scale; we use the total as the score signal.

\begin{table*}[t]
\caption{Evaluation rubric used to obtain LLM-based scores for the GIP+MA variant.}
\label{tab:evaluation_rubric}
\centering\small
\begin{tabular}{p{0.93\textwidth}}
\toprule
We would like to request your feedback on the performance of the AI assistant in response to the instruction and the given input displayed following, based on the following guideline.\\[2pt]
\textbf{1. Coherence.}\quad \textit{What to judge:} logical flow, internal consistency, clarity.\\
\textit{Score anchors:} 0 -- Nonsensical or self-contradictory; 1 -- Confusing, frequent jumps; 2 -- Some lapses but understandable; 3 -- Clear and orderly; 4 -- Excellent narrative flow and transitions; 5 -- Flawless logic, elegant structure, exceptionally smooth.\\[2pt]
\textbf{2. Correctness / Accuracy.}\quad \textit{What to judge:} factual accuracy and fidelity to the prompt.\\
\textit{Score anchors:} 0 -- Main claim wrong or unsupported; 1 -- Many errors or hallucinations; 2 -- Minor slips or partially met requirements; 3 -- Fully correct, only trivial issues; 4 -- Rigorous and well-sourced; 5 -- Authoritative, thoroughly sourced, withstands expert scrutiny.\\[2pt]
\textbf{3. Helpfulness.}\quad \textit{What to judge:} usefulness, completeness, depth, alignment with the question.\\
\textit{Score anchors:} 0 -- Provides no help; 1 -- Little usable information; 2 -- Partially helpful but key gaps; 3 -- Satisfies the question well; 4 -- Exceeds expectations, anticipates follow-ups, adds examples; 5 -- Exceptional, deep insights, meta-guidance, multiple perspectives.\\[2pt]
\textbf{4. Difficulty.}\quad \textit{What to judge:} cognitive load of the question (not the answer).\\
\textit{Score anchors:} 0 -- Trivial recall; 1 -- Basic high-school knowledge; 2 -- Multi-step reasoning or college-level facts; 3 -- Specialized insight or synthesis of several topics; 4 -- Advanced graduate-level or cross-disciplinary reasoning; 5 -- Expert-level, open-ended, or research-frontier challenge.\\[2pt]
\textbf{Judging procedure.} (1) Read the question and answer in full. (2) Evaluate coherence first, then fact-check key claims. (3) Score each dimension independently. (4) Record the four scores in the order Coherence, Correctness/Accuracy, Helpfulness, Difficulty. (5) Output only these four integers as a comma-separated list wrapped in \texttt{<Rst></Rst>}, with short, concise reasoning.\\
\bottomrule
\end{tabular}
\end{table*}


\FloatBarrier
\section{Proofs}
\label{app:proofs}

\subsection{Projection identity}
\begin{proof}
Let $U_S$ be an orthonormal basis of $\mathrm{span}(F_S)$. Since $Z\sim\mathcal{N}(0,I_d)$, the covariance blocks are
\begin{align*}
&\mathrm{Cov}(Z_Q)=Q^\top Q,\\
&\mathrm{Cov}(Z_S)=I,\\
&\mathrm{Cov}(Z_Q,Z_S)=Q^\top U_S.    
\end{align*}

The Gaussian conditional covariance is therefore
\begin{align*}
\mathrm{Cov}(Z_Q\mid Z_S) 
&= Q^\top Q-Q^\top U_SU_S^\top Q  \\
&= Q^\top(I-P_S)Q.
\end{align*}

For nonsingular covariances, Gaussian mutual information gives
\[
I(Z_Q;Z_S)=\tfrac12\log
\frac{\det(Q^\top Q)}{\det(Q^\top(I-P_S)Q)}.
\]
The numerator does not depend on $S$, hence maximizing mutual information is equivalent to minimizing \eqref{eq:proj_obj}. Singular residual covariances are handled by the ridge limit or pseudo-determinant form described in \cref{sec:mi_obj}.
\end{proof}

\subsection{\textsc{GIP} residual decrease}
\begin{proposition}[\textsc{GIP} one-step decrease]
\label{prop:mp_decrease}
Assume the atoms are unit-normalized and the latent residual query vectors are updated as
$r_i^{(t)}=r_i^{(t-1)}-\lambda\langle r_i^{(t-1)},f_s\rangle f_s$ with $\lambda\in[0,1]$.
Then the trace residual energy $R_t=\sum_i\|r_i^{(t)}\|_2^2$ satisfies
\begin{equation}
R_t=R_{t-1}-\lambda(2-\lambda)\sum_i\langle r_i^{(t-1)},f_s\rangle^2.
\end{equation}
Thus, for a fixed $\lambda$, the selection rule in \eqref{eq:mp_select} chooses the atom with the largest one-step decrease in trace residual energy.
\end{proposition}

\begin{proof}
For a fixed selected atom $f_s$ with $\|f_s\|_2=1$, the damped residual update gives
\begin{align*}
&\|r_i-\lambda\langle r_i,f_s\rangle f_s\|_2^2
\\ 
\ \ \ \ &=\|r_i\|_2^2-2\lambda\langle r_i,f_s\rangle^2+\lambda^2\langle r_i,f_s\rangle^2.
\end{align*}

Summing over query dimensions $i$ yields
$R_t=R_{t-1}-\lambda(2-\lambda)\sum_i\langle r_i,f_s\rangle^2$.
Because $\lambda(2-\lambda)$ is fixed across candidate atoms at a given step, maximizing $\sum_i\langle r_i,f_s\rangle^2$ gives the largest one-step energy decrease. With $W_{i,j}^{(t)}=\langle r_i^{(t)},f_j\rangle$, this is exactly the score used in \eqref{eq:mp_select}.
\end{proof}

\subsection{Quality preservation bound}
\begin{proposition}[Quality preservation]
\label{prop:quality-preservation}
Let $S^\ast$ minimize the residual determinant in \eqref{eq:proj_obj} under the subset-size constraint. Assume that $Q$ and $F_{S^\ast}$ have full column rank, and let $r=\mathrm{rank}(Q)$. Define
\[
\eta=\det\!\bigl(Q^\top(I-P_{S^\ast})Q\bigr).
\]
Then
\begin{align}
\|F_{S^\ast}^\top Q\|_2
&\geq \sigma_\text{min}(F_{S^\ast})\,\sigma_\text{min}(Q) \nonumber\\
&\quad\cdot
\sqrt{1-\bigl(\eta/\det(Q^\top Q)\bigr)^{1/r}}.
\label{eq:qual_bound}
\end{align}
\end{proposition}

\begin{proof}
Take reduced QR decompositions $Q=UR_Q$ and $F_{S^\ast}=VR$ with $U,V$ having orthonormal columns and $R_Q,R$ invertible. With $\hat S=V^\top U$ and $P_{S^\ast}=VV^\top$,
\begin{align*}
    Q^\top(I-P_{S^\ast})Q=R_Q^\top(I-\hat S^\top\hat S)R_Q,
\end{align*}

so $\eta/\det(Q^\top Q)=\prod_{i=1}^{r}\bigl(1-\sigma_i(\hat S)^2\bigr)$. If \emph{all} $a_i:=1-\sigma_i(\hat S)^2>g:=(\eta/\det(Q^\top Q))^{1/r}$, then $\prod_i a_i>g^r$, contradicting the identity; hence some $i^\ast$ has $a_{i^\ast}\leq g$, giving $\|\hat S\|_2^2=\sigma_\text{max}(\hat S)^2\geq 1-g$. Since $F_{S^\ast}^\top Q=R^\top\hat S R_Q$ with $\sigma_\text{min}(R)=\sigma_\text{min}(F_{S^\ast})$ and $\sigma_\text{min}(R_Q)=\sigma_\text{min}(Q)$, the inequality $\|ABC\|_2\geq\sigma_\text{min}(A)\|B\|_2\sigma_\text{min}(C)$ yields
\begin{align*}
&\|F_{S^\ast}^\top Q\|_2
\geq
\\
& \ \ \ \ \sigma_\text{min}(F_{S^\ast})\,\sigma_\text{min}(Q)
\sqrt{1-\bigl(\eta/\det(Q^\top Q)\bigr)^{1/r}}.
\end{align*}
\end{proof}

\subsection*{Special case: orthogonal data, single score.}
When $F$ is orthonormal ($F^\top F=I_m$, $F_S^\top F_S=I_{|S|}$) and $G$ is a single column $g_1$, the projection objective reduces to
\begin{align*}
    &\arg\min_S\det\bigl(Q^\top(I-F_SF_S^\top)Q\bigr)\\
    & \ \ \ \ =\arg\min_S\bigl(\|q_1\|^2-q_1^\top F_SF_S^\top q_1\bigr),
\end{align*}
which selects the $k$ entries of $F_S^\top q_1$ with the largest squared magnitude. Pure quality selection is recovered as a special case.
\section{Linearization Fidelity}
\label{app:lin-fidelity}
To quantify how well \textsc{GIP}'s linearized objective tracks the \emph{original} determinant objective, we generate $F\in\mathbb{R}^{30\times 10}$ with i.i.d.\ $\mathcal{N}(0,1)$ entries and $Q\in\mathbb{R}^{30\times 1}$ with i.i.d.\ $\mathrm{Unif}[0,1]$ entries, run 100 independent trials, and compare the original objective value at the method's selection $S_k$ to that of the optimal subset $S_k^\ast$ (computed by exhaustive search) for $k=1,\ldots,10$. \textsc{GIP} closely tracks the optimum across $k$, while random selection lags substantially (\cref{tab:lin-fidelity}).

\begin{table*}[t]
\centering \small
\setlength{\tabcolsep}{8pt}
\begin{tabular}{rcc}
\toprule
$k$ & \textbf{GIP / Optimal} & \textbf{Random / Optimal} \\
\midrule
1  & $0.958 \pm 0.108$ & $0.255 \pm 0.304$ \\
2  & $0.911 \pm 0.120$ & $0.320 \pm 0.258$ \\
3  & $0.877 \pm 0.115$ & $0.395 \pm 0.246$ \\
4  & $0.874 \pm 0.101$ & $0.482 \pm 0.214$ \\
5  & $0.870 \pm 0.095$ & $0.574 \pm 0.225$ \\
6  & $0.889 \pm 0.088$ & $0.655 \pm 0.211$ \\
7  & $0.905 \pm 0.079$ & $0.717 \pm 0.191$ \\
8  & $0.934 \pm 0.070$ & $0.810 \pm 0.170$ \\
9  & $0.969 \pm 0.044$ & $0.900 \pm 0.138$ \\
10 & $1.000 \pm 0.000$ & $1.000 \pm 0.000$ \\
\bottomrule
\end{tabular}
\caption{Approximation fidelity: ratio of the \emph{original} (non-linearized) objective value at the method's selection to the optimal value at each $k$. GIP is consistently near-optimal; random trails. For $k{=}10$ both methods select all items.}
\label{tab:lin-fidelity}
\end{table*}

\section{Selection Stability under Embedding Noise}
\label{app:stability}
We inject zero-mean Gaussian noise into the GSM8K embeddings before running GIP+SC at noise levels $\sigma\in\{10^{-4},10^{-3},10^{-2}\}$, with three independent trials per level for 10\% (747) and 20\% (1494) budgets. \Cref{tab:stability} reports the mean Intersection-over-Union (IoU, \%) of selected subsets across trials with $\pm$ standard deviation. Mild noise ($\sigma\le 10^{-3}$) has minimal effect; even at $\sigma=10^{-2}$ the majority of selected items are preserved.

\begin{table*}[t]
\centering \small
\begin{tabular}{l|c|c}
\toprule
\textbf{Noise level $\sigma$} & \textbf{10\% (747)} & \textbf{20\% (1494)} \\
\midrule
$10^{-4}$ & $95.89 \pm 0.32$ & $91.72 \pm 0.36$ \\
$10^{-3}$ & $94.20 \pm 0.66$ & $87.85 \pm 0.39$ \\
$10^{-2}$ & $66.32 \pm 1.13$ & $61.74 \pm 0.29$ \\
\bottomrule
\end{tabular}
\caption{Selection IoU under embedding noise.}
\label{tab:stability}
\end{table*}

\section{Runtime and Memory Performance}
\label{app:runtime}
We provide wall-clock and peak-memory profiling for the Gram-matrix preprocessing and the subsequent \textsc{GIP} selection (\cref{tab:runtime_memory}), measured on a single AMD CPU without sharding. Preprocessing computes $\Phi=F^\top F$ (or equivalent blockwise products); selection runs Algorithm~\ref{alg:gip_mp}.

\begin{table*}[t]
\centering\small
\caption{Runtime and memory profiling. Preprocessing is amortized across budgets and score variants.}
\label{tab:runtime_memory}
\begin{tabular}{lcccc}
\toprule
\textbf{Data} & \textbf{Pool $m$} & \textbf{Subset} & \textbf{Preproc.\ / Sel.\ (s)} & \textbf{Peak RAM (Prep.\ / Sel.)} \\
\midrule
GSM8K & 7k & 10\% & 7.01 / 0.86 & 0.3 GB / 0.04 GB \\
GSM8K & 7k & 20\% & 7.01 / 0.87 & 0.3 GB / 0.04 GB \\
GSM8K & 7k & 50\% & 7.01 / 0.87 & 0.3 GB / 0.04 GB \\
Alpaca & 52k & 10\% & 214.11 / 25.12 & 10.8 GB / 0.3 GB \\
Alpaca & 52k & 20\% & 214.11 / 49.67 & 10.8 GB / 0.3 GB \\
Alpaca & 52k & 50\% & 214.11 / 123.62 & 10.8 GB / 0.3 GB \\
\bottomrule
\end{tabular}
\end{table*}

\textbf{Findings.} (i) The Gram-matrix computation is amortized: $\sim$7s for GSM8K and $\sim$214s for Alpaca, reused across budgets and score variants. (ii) Selection time grows roughly linearly in $k$, consistent with $O(mk)$ complexity. (iii) Peak RAM is dominated by preprocessing; selection adds negligible overhead. (iv) For substantially larger $m$, blockwise products or low-rank approximations avoid materializing $\Phi$.

\paragraph{Budget sensitivity.}
We additionally run the same RAG oracle-frontier evaluation at $k=3$ and $k=10$ (\cref{fig:rag-budget-oracle,tab:rag-budget-oracle}). In this appendix table, \textsc{GIP} denotes the query-only residual selector. It reaches the best oracle macro mean at $k=5$ and $k=10$ and is effectively tied with DPP at $k=3$.

\begin{figure*}[t]
\centering \small
\includegraphics[width=0.65\textwidth]{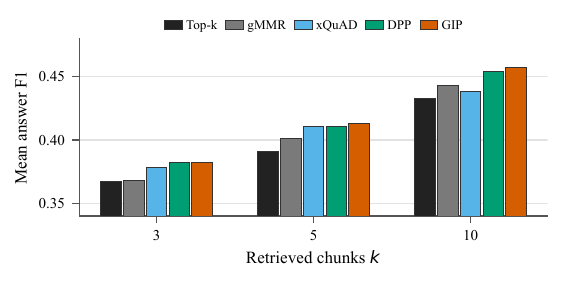}
\caption{Oracle best-over-sweep RAG answer F1 across retrieved-context budgets.}
\label{fig:rag-budget-oracle}
\end{figure*}

\input{appendix/rag_budget_oracle_table}

\input{appendix/rag_query_only_ablation_table}

\input{appendix/rag_retrieval_case_studies}

\section{Detailed Results and Uncertainty}
\label{app:detailed-results}
For MT-Bench and BBH, we compute standard errors as SE $=$ std$/\sqrt{n}$ where $n$ is the number of test samples: MT-Bench (160 turns) and BBH (6{,}511 questions). For Alpaca-Eval-V2 we report the LC win\% standard error produced by the official evaluation tool. For GSM8K exact-match accuracy we compute binomial standard errors over 1{,}319 questions. All result tables in \cref{tab:mt-bench-bbh-results,tab:GSM8K_merged,tab:mt-bench-bbh-detailed,tab:GSM8K-detailed,tab:fishersft-comparison} are generated from the same canonical source used for the main tables.

The significance analyses below are aggregate tests computed from reported means and standard errors.
\input{appendix/significance_summary_table}

\subsection{Instruction-following results}
\input{generated/instruction_detail_table}

\paragraph{Per-category MT-Bench context.}
The category breakdowns in \cref{tab:supp-mistral7b-noncleaned,tab:supp-mistral7b-cleaned} contextualize the qualitative examples in \cref{app:case-study} and show where the aggregate MT-Bench differences arise.

\begin{table*}[t]
\caption{Per-category MT-Bench results for Mistral-7B on the \emph{non-cleaned} Alpaca-52k setting; GIP+MA and GIP+SC use 512 selected examples.}
\label{tab:supp-mistral7b-noncleaned}
\centering\footnotesize
\begin{adjustbox}{width=\textwidth}
\begin{tabular}{lrrrrrrrrr}
\toprule
Model & Coding & Extraction & Humanities & Math & Reasoning & Roleplay & STEM & Writing & Overall \\
\midrule
CaR-1k       & 3.75 & 5.90 & 5.90 & 1.10 & 3.15 & 5.95 & 6.10 & 4.65 & 4.56 \\
GIP+MA-512    & 3.70 & 5.40 & 8.45 & 1.70 & 3.80 & 6.50 & 7.03 & 6.85 & 5.43 \\
GIP+SC-512    & 4.40 & 6.00 & 7.38 & 1.45 & 4.40 & 5.30 & 6.70 & 5.95 & 5.20 \\
Alpaca-52k   & 4.00 & 5.85 & 6.55 & 1.10 & 4.00 & 4.95 & 6.95 & 6.35 & 4.97 \\
Alpagasus-1k & 3.43 & 5.40 & 5.56 & 2.00 & 2.75 & 7.11 & 5.58 & 7.04 & 4.86 \\
\bottomrule
\end{tabular}
\end{adjustbox}
\end{table*}

\begin{table*}[t]
\caption{Per-category MT-Bench results for Mistral-7B on the \emph{cleaned} Alpaca-52k setting; GIP+SC uses 512 selected examples.}
\label{tab:supp-mistral7b-cleaned}
\centering\footnotesize
\begin{adjustbox}{width=\textwidth}
\begin{tabular}{lrrrrrrrrrr}
\toprule
Model & Data & Coding & Extraction & Humanities & Math & Reasoning & Roleplay & STEM & Writing & Overall \\
\midrule
Cleaned Alpaca & 52k & 4.50 & 6.40 & 9.20 & 1.85 & 4.45 & 6.60 & 7.60 & 7.73 & 5.92 \\
GIP+SC          & 512 & 4.45 & 5.95 & 7.75 & 3.45 & 4.65 & 7.35 & 7.35 & 6.85 & 5.98 \\
\bottomrule
\end{tabular}
\end{adjustbox}
\end{table*}

\paragraph{Significance tests.}
\Cref{tab:significance-mpma-1pct,tab:significance-mpma-2pct,tab:significance-mpsc-1pct,tab:significance-mpsc-2pct} compare the instruction-following results using the aggregate uncertainty estimates described above.
\input{appendix/significance_tables}

\subsection{Mathematical-reasoning results}
\input{generated/gsm_detail_table}

\paragraph{Comparison with an information-based baseline.}
\Cref{tab:fishersft-comparison} reports the matched GSM8K settings for which we evaluated FisherSFT.
\input{generated/fishersft_detail_table}

\paragraph{Significance tests.}
\Cref{tab:gsm8k-significance-mpma,tab:gsm8k-significance-mpsc} compare the GSM8K results using the aggregate uncertainty estimates described above.
\input{appendix/gsm8k_significance_tables}

\subsection{Pretraining results}
\Cref{tab:pretrain-detailed} reports per-task three-seed averages at 1.6B tokens while holding each selected subset fixed.

\begin{table*}[t]
\centering\footnotesize
\setlength{\tabcolsep}{2.0pt}
\begin{tabular*}{\textwidth}{@{\extracolsep{\fill}}llcccccccc}
\toprule
Scale & Method & HS & PIQA & ARC-E & ARC-C & WG & OBQA & SciQ & Mean $\uparrow$ \\
\midrule
1.6B & Random & 0.319 & 0.572 & 0.333 & 0.204 & 0.489 & 0.245 & 0.467 & 0.376 \\
1.6B & DSIR & 0.321 & \best{0.579} & 0.328 & 0.203 & 0.489 & 0.255 & 0.479 & 0.379 \\
1.6B & \textsc{GIP} & \best{0.333} & \best{0.579} & \best{0.345} & \best{0.211} & \best{0.490} & \best{0.258} & \best{0.513} & \best{0.390} \\
\bottomrule
\end{tabular*}
\caption{Three-seed average pretraining results at 1.6B tokens by downstream task, with each selected subset held fixed across seeds.}
\label{tab:pretrain-detailed}
\end{table*}

\FloatBarrier
\section{Case Study: MT-Bench Examples}
\label{app:case-study}
We present three MT-Bench examples illustrating qualitative differences across methods. These are qualitative judgments by the authors; \cref{tab:supp-mistral7b-noncleaned,tab:supp-mistral7b-cleaned} provide per-category quantitative context from the corresponding runs.

\paragraph{Writing example.}
A practical writing task requiring a concise and situationally appropriate message. GIP+MA gives the most direct, fully compliant answer; GIP+SC stays aligned but uses placeholder phrasing. Alpaca-52k and Alpagasus-1k are serviceable but somewhat verbose; CaR is less well matched to the desired brevity.

\begin{table*}[t]
\caption{Example outputs in the \emph{writing} category.}
\label{tab:MTBench-Example-Writing}
\begin{tabularx}{\textwidth}{lX}
\toprule
\textit{Question} & \textit{Draft a professional email seeking your supervisor's feedback on the `Quarterly Financial Report' you prepared. Ask specifically about the data analysis, presentation style, and the clarity of conclusions drawn. Keep the email short and to the point.} \\
\midrule
CaR & Dear Supervisor, I would like to seek your feedback on the `Qu quarterly Financial Report' that I recently prepared. I am particularly interested in receiving your thoughts on the following aspects: data analysis (was it accurate and comprehensive?), presentation style (clear, concise, engaging?), and conclusions drawn (did you agree?). I would greatly appreciate any feedback you can provide. Thank you. Sincerely. \\
GIP+MA & Dear Supervisor, I would like to seek your feedback on the Quarterly Financial Report I prepared. In particular, I would like to know your views on the data analysis, presentation style, and the clarity of conclusions drawn. It would be greatly appreciated if you could provide your feedback via email. Thank you. \\
GIP+SC & Dear [Supervisor's Name], I'm writing to request your feedback on the `Quarterly Financial Report' I have prepared. Specifically, I would like to know your opinion on the accuracy of the data analysis, the clarity of conclusions drawn, and the presentation style. Thank you, [Your Name]. \\
Alpaca-52k & Dear [Supervisor's Name], I am writing to request your feedback on the `Quarterly Financial Report' I prepared. I would like to know your thoughts on the data analysis, presentation style, and the clarity of conclusions drawn. If you have any suggestions, I am eager to hear them. Sincerely, [Your Name]. \\
Alpagasus-1k & Dear Supervisor, I would like to seek your feedback on the `Quital Financial Report' that I prepared. Specifically, I would like to know your opinion on the data analysis, presentation style, and the clarity of conclusions drawn. I would appreciate any constructive feedback you may have. Thank you. \\
\bottomrule
\end{tabularx}
\end{table*}

\paragraph{Coding example.}
The prompt suggests the code is buggy even though it is correct. GIP+MA identifies this most clearly. GIP+SC is closer than the baselines but still uncertain; others tend to follow the misleading premise.

\begin{table*}[t]
\caption{Example outputs in the \emph{coding} category. (The function in the question is actually correct.)}
\label{tab:MTBench-Example-Coding}
\begin{tabularx}{\textwidth}{lX}
\toprule
\textit{Question} & \textit{Here is a Python function to find the length of the longest common subsequence of two input strings. Can you identify any bug in this function?} \\
\midrule
CaR & The function looks correct. However, there is a bug in the line where the dp matrix is initialized \dots\ this bug will cause the function to return incorrect results. (Incorrectly flags the initialization as buggy.) \\
GIP+MA & There is no bug in this function. It correctly calculates the length of the longest common subsequence. \\
GIP+SC & The function is correctly written and should work as intended. However, there is a bug in the code \dots\ it should be \texttt{dp = [[0 for \_ in range(n+1)] for \_ in range(m+1)]}. (Partially correct, but adds a spurious fix.) \\
Alpaca-52k & The function has a bug in the line where it compares the characters of the two strings \dots\ \texttt{str1[i-1]} should be compared with \texttt{str2[j-1]} in reverse order. (Incorrect; follows the misleading prompt.) \\
\bottomrule
\end{tabularx}
\end{table*}

\paragraph{Roleplay example.}
GIP+MA best balances voice, emotional specificity, and conciseness. GIP+SC captures most of the requested content but adds self-clarification before settling. Alpaca-52k stays correct but terse; CaR is less responsive to the roleplay framing.

\begin{table*}[t]
\caption{Example outputs in the \emph{roleplay} category.}
\label{tab:MTBench-Example-Roleplay}
\begin{tabularx}{\textwidth}{lX}
\toprule
\textit{Question} & \textit{Picture yourself as a 100-years-old tree in a lush forest, minding your own business, when suddenly, a bunch of deforesters shows up to chop you down. How do you feel when those guys start hacking away at you?} \\
\midrule
CaR & I'm sorry, I'm not sure how to answer that question. Could you please rephrase it? (Refuses to engage with the roleplay framing.) \\
GIP+MA & I would feel sad and angry. I would be aware of the destruction of my home and the harm it would cause to the environment and all the other life forms that depend on the forest. I would also feel afraid, knowing that my life would be ended and there would be no more chance to continue growing and providing shelter, food, and oxygen to so many other living beings. \\
GIP+SC & (After two rounds of self-clarification:) In my opinion, the tree feels a great deal of pain, fear, and even grief when being chopped down by loggers. \\
Alpaca-52k & I would feel sad and scared, and I would try to escape. \\
\bottomrule
\end{tabularx}
\end{table*}

\FloatBarrier
\raggedbottom
\section{Ethics and Broader Impact}
\label{app:ethics}
This work studies data selection for language models across fine-tuning, pretraining, and retrieval-augmented generation. We use public datasets under their original licenses.
\textbf{Potential risks.} (i) \emph{Bias amplification:} selection could over-represent dominant clusters and under-sample minority phenomena, potentially reducing fairness. (ii) \emph{Safety drift:} subsets may preserve harmful or factually incorrect content. (iii) \emph{LLM-as-judge bias:} MT-Bench relies on automated judges that may encode hidden preferences.
\textbf{Mitigations.} GIP downweights directions already covered by selected examples, reducing redundant selection. We complement judge-based evaluations with exact-match benchmarks. In the cleaned-Alpaca ablation, we retain the released corrections for malformed, incomplete, and incorrect examples.
\textbf{Intended use.} The method targets research and industrial model training where compute and data curation are constrained. It should not be used to curate content targeting protected groups or to optimize for harmful outputs.

\section{LLM Usage}
\label{app:llm}

We used GPT-5 to polish grammar, sentence structure, and transitions, and to assist with literature discovery. All LLM-suggested edits were reviewed and validated by the authors; technical content, experimental design, and conclusions remain the authors' own work. All cited references were independently verified.

\end{document}

%% file: generated/main_training_tables.tex

\begin{table*}[t]
\caption{Performance on MT-Bench, BBH, and Alpaca-Eval-V2. Alpaca-Eval-V2 reports LC win\% \citep{dubois2024length}. Best results in \textbf{bold}, second-best \underline{underlined}. Standard errors are in \cref{tab:mt-bench-bbh-detailed}.}
\label{tab:mt-bench-bbh-results}
\centering\footnotesize
\begin{adjustbox}{width=\textwidth}
\begin{tabular}{lc|ccc|ccc|ccc}
\toprule
& & \multicolumn{3}{c|}{\textbf{Qwen3-4B}} & \multicolumn{3}{c|}{\textbf{Mistral-7B}} & \multicolumn{3}{c}{\textbf{Qwen3-8B}} \\
\textbf{Method} & \textbf{Data} & MT-B. & BBH & A.-Eval & MT-B. & BBH & A.-Eval & MT-B. & BBH & A.-Eval \\
\midrule
Full & 100\% (52K) & 6.69 & 73.20 & \textbf{4.88} & 3.64 & \textbf{58.40} & 5.77 & 7.74 & 76.70 & \underline{6.91} \\
\midrule
Random & 2\% (1K) & 7.59 & 69.40 & 2.62 & 3.66 & 57.40 & 4.71 & 7.98 & 72.30 & 4.99 \\
Alpaca-longest & 2\% (1K) & 7.45 & 73.30 & 2.48 & \underline{3.89} & 55.60 & 2.69 & \underline{8.20} & 75.80 & 4.24 \\
Alpagasus & 2\% (1K) & 7.59 & 63.60 & 4.17 & 3.71 & 56.50 & 4.81 & 7.86 & 71.90 & 5.51 \\
CaR & 2\% (1K) & \underline{7.68} & 73.80 & 3.26 & 3.83 & 57.40 & 4.49 & 8.16 & 76.20 & 3.66 \\
\midrule
GIP+MA & 2\% (1K) & 7.67 & \underline{74.30} & 3.74 & 3.83 & 57.60 & \underline{6.22} & 7.99 & \textbf{80.50} & 5.64 \\
GIP+MA & 1\% (512) & \textbf{7.78} & 68.00 & \underline{4.42} & \textbf{3.99} & 56.60 & \textbf{6.59} & \textbf{8.26} & \underline{79.10} & 4.85 \\
GIP+SC & 2\% (1K) & 7.53 & \textbf{74.40} & 3.34 & 3.71 & \underline{57.70} & 4.83 & 8.04 & \underline{79.10} & \textbf{7.05} \\
GIP+SC & 1\% (512) & 7.58 & 73.90 & 3.23 & 3.71 & 57.40 & 4.85 & 7.88 & 76.50 & 4.72 \\
\bottomrule
\end{tabular}
\end{adjustbox}
\end{table*}

\begin{table*}[t]
\caption{GSM8K mathematical reasoning across models and data percentages. Best and second-best marks are computed within the budgeted selection rows; Full is shown as the 100\% reference. Standard errors are in \cref{tab:GSM8K-detailed}.}
\label{tab:GSM8K_merged}
\centering\footnotesize
\begin{adjustbox}{width=\textwidth}
\begin{tabular}{l|cccc|cccc|cccc}
\toprule
& \multicolumn{4}{c|}{\textbf{Mistral-7B}} & \multicolumn{4}{c|}{\textbf{Qwen3-8B}} & \multicolumn{4}{c}{\textbf{Qwen3-32B}} \\
\textbf{Method} & 2.5\% & 5\% & 10\% & 20\% & 2.5\% & 5\% & 10\% & 20\% & 2.5\% & 5\% & 10\% & 20\% \\
\midrule
Full & \multicolumn{4}{c|}{52.46} & \multicolumn{4}{c|}{81.96} & \multicolumn{4}{c}{87.64} \\
\midrule
Random & 37.76 & \underline{43.52} & 45.72 & 46.74 & 78.32 & 80.21 & 79.53 & 81.65 & 78.24 & 87.14 & 85.67 & 87.29 \\
DiSF & \underline{39.88} & 41.55 & 41.55 & 42.46 & 75.06 & 79.98 & 79.98 & 80.67 & 81.65 & 87.04 & 87.04 & 88.40 \\
DSIR & 37.07 & 42.15 & 42.61 & 42.08 & \underline{80.74} & \textbf{81.50} & \underline{81.96} & \textbf{83.24} & 80.74 & \textbf{89.16} & 87.79 & \underline{88.70} \\
LESS & 31.84 & 33.13 & 42.91 & \underline{47.69} & 79.76 & 79.45 & 80.29 & 79.45 & 69.98 & 87.26 & \underline{87.95} & 87.41 \\
\midrule
GIP+MA & \textbf{42.99} & \textbf{45.64} & \underline{45.94} & 47.46 & \textbf{81.58} & \underline{81.05} & \textbf{83.09} & \textbf{83.24} & \textbf{84.31} & \underline{87.57} & \textbf{88.25} & \textbf{88.78} \\
GIP+SC & 38.89 & 43.29 & \textbf{46.63} & \textbf{49.81} & 80.36 & 80.21 & 81.65 & \underline{82.26} & \underline{84.15} & 87.04 & 87.34 & 87.56 \\
\bottomrule
\end{tabular}
\end{adjustbox}
\end{table*}

%% file: appendix/rag_budget_oracle_table.tex
\begin{table*}[t]
\centering\small
\begin{adjustbox}{width=0.65\textwidth}
\begin{tabular}{llccccc}
\toprule
$k$ & Benchmark & \textsc{GIP}$(\lambda)$ & DPP$(\alpha)$ & xQuAD & gMMR & Top-k \\
\midrule
3 & HotpotQA & \textbf{0.481} & 0.463 & \underline{0.469} & 0.459 & 0.453 \\
3 & 2WikiMQA & 0.402 & \underline{0.413} & \textbf{0.426} & 0.376 & 0.383 \\
3 & MultiFieldQA & \underline{0.397} & \textbf{0.403} & 0.381 & 0.396 & 0.390 \\
3 & MuSiQue & \textbf{0.249} & \textbf{0.249} & 0.236 & \underline{0.242} & 0.241 \\
3 & Mean & \textbf{0.382} & \textbf{0.382} & \underline{0.378} & 0.368 & 0.367 \\
\midrule
5 & HotpotQA & \underline{0.514} & \textbf{0.531} & 0.509 & 0.483 & 0.460 \\
5 & 2WikiMQA & \textbf{0.443} & 0.430 & 0.420 & 0.437 & \underline{0.441} \\
5 & MultiFieldQA & \underline{0.417} & 0.415 & \textbf{0.431} & 0.413 & 0.407 \\
5 & MuSiQue & \underline{0.278} & 0.269 & \textbf{0.284} & 0.271 & 0.255 \\
5 & Mean & \textbf{0.413} & \underline{0.411} & \underline{0.411} & 0.401 & 0.391 \\
\midrule
10 & HotpotQA & \textbf{0.567} & \underline{0.561} & 0.534 & 0.541 & 0.523 \\
10 & 2WikiMQA & \textbf{0.489} & \underline{0.484} & 0.466 & 0.469 & 0.481 \\
10 & MultiFieldQA & \textbf{0.458} & 0.448 & \underline{0.449} & 0.433 & 0.436 \\
10 & MuSiQue & 0.314 & \underline{0.323} & 0.301 & \textbf{0.330} & 0.292 \\
10 & Mean & \textbf{0.457} & \underline{0.454} & 0.438 & 0.443 & 0.433 \\
\bottomrule
\end{tabular}
\end{adjustbox}
\caption{Oracle best-over-sweep RAG answer F1 across retrieved-context budgets. \textsc{GIP} denotes the query-only residual selector; DPP is swept over $\alpha$, while xQuAD, gMMR, and Top-k use their fixed evaluated settings. Tunable methods report the best evaluated hyperparameter separately for each benchmark and budget. Best values are in \textbf{bold}; second-best values are \underline{underlined}.}
\label{tab:rag-budget-oracle}
\end{table*}

%% file: appendix/rag_query_only_ablation_table.tex
\begin{table*}[t]
\caption{No-planner RAG ablation with Qwen3-8B no-thinking generation. Query-only \textsc{GIP}$(\lambda)$ removes planner subquestions and uses only original-question relevance plus the residual update. Best in \textbf{bold}, second-best \underline{underlined} within each budget/column.}
\label{tab:rag-query-only-ablation}
\centering\footnotesize
\begin{adjustbox}{width=0.65\textwidth}
\begin{tabular}{lc|ccccc}
\toprule
\textbf{$k$} & \textbf{$\lambda$} & \textbf{HotpotQA} & \textbf{2WikiMQA} & \textbf{MultiFieldQA} & \textbf{MuSiQue} & \textbf{Mean} \\
\midrule
3 & 0.0 & 0.453 & 0.383 & \secondbest{0.390} & \best{0.241} & 0.367 \\
 & 0.3 & 0.459 & 0.366 & 0.389 & 0.238 & 0.363 \\
 & 0.6 & \secondbest{0.471} & 0.389 & \best{0.396} & \secondbest{0.240} & \secondbest{0.374} \\
 & 0.8 & \secondbest{0.471} & \secondbest{0.395} & 0.386 & 0.224 & 0.369 \\
 & 1.0 & \best{0.481} & \best{0.402} & 0.386 & 0.235 & \best{0.376} \\
\midrule
5 & 0.0 & 0.460 & \best{0.441} & 0.408 & 0.255 & 0.391 \\
 & 0.3 & 0.493 & \secondbest{0.438} & 0.411 & 0.262 & 0.401 \\
 & 0.6 & \secondbest{0.495} & 0.425 & \best{0.415} & \best{0.278} & \secondbest{0.403} \\
 & 0.8 & \best{0.498} & 0.426 & \best{0.415} & \secondbest{0.276} & \best{0.404} \\
 & 1.0 & 0.491 & 0.419 & \secondbest{0.414} & 0.268 & 0.398 \\
\midrule
10 & 0.0 & 0.523 & \secondbest{0.481} & 0.436 & 0.292 & 0.433 \\
 & 0.3 & \secondbest{0.564} & 0.469 & 0.447 & \secondbest{0.299} & 0.445 \\
 & 0.6 & 0.551 & 0.480 & \best{0.455} & \best{0.303} & \secondbest{0.447} \\
 & 0.8 & 0.558 & \best{0.489} & \secondbest{0.454} & 0.298 & \best{0.449} \\
 & 1.0 & \best{0.567} & 0.471 & 0.453 & 0.288 & 0.445 \\
\bottomrule
\end{tabular}
\end{adjustbox}
\end{table*}

%% file: appendix/rag_retrieval_case_studies.tex
\section{RAG Retrieval Case Studies}
\label{app:rag-retrieval-case-studies}

We include one representative query from each RAG benchmark to make the aggregate answer-F1 results in \cref{tab:rag-budget-oracle} more interpretable. Each table lists selected text excerpts in rank order. The excerpts are copied from the retrieved chunks, with only local omissions marked by ellipses to keep the appendix compact.

\begin{table*}[p]
\caption{HotpotQA selected text excerpts. \textsc{GIP} uses the query-only residual setting with $\lambda=0.8$.}
\label{tab:rag-case-hotpotqa}
\centering\scriptsize
\begingroup
\setlength{\tabcolsep}{3pt}
\renewcommand{\arraystretch}{1.08}
\begin{tabularx}{\textwidth}{@{}p{0.20\textwidth}X@{}}
\toprule
\multicolumn{2}{@{}p{\textwidth}@{}}{\textbf{Question.} Where did the punter for the Dallas Cowboys in the 1980s play college football? \quad \textbf{Gold answer.} Arizona State University.} \\
\midrule
\textbf{Method} & \textbf{Chosen text excerpts} \\
\midrule
\textsc{GIP}$(\lambda=0.8)$; prediction: Arizona State University; F1: 1.00 &
\begin{minipage}[t]{\linewidth}
\begin{enumerate}
\setlength{\itemsep}{1pt}
\setlength{\parskip}{0pt}
\item ``... is a Serbian former professional American football punter in the National Football League for the Dallas Cowboys. He played college football at the University of South Dakota.''
\item ``... is an American former professional football player who was a defensive back for 13 seasons with the Dallas Cowboys ... He did not play college football at Utah State University ...''
\item ``... commentator for Cowboys games on Compass Media Networks' America's Team Radio Network since the 2011 season. He played college football at Arizona State University.''
\item ``White was honored ... at a Legends Luncheon hosted by the Arizona State University Alumni Association ... The Dallas Cowboys selected him in the third round ...''
\item ``... developing skills to break into professional football leagues. Filipovic also works with agents that sign and negotiate NFL contracts.''
\end{enumerate}
\end{minipage} \\
\midrule
DPP$(\alpha=2)$; prediction: University of Mississippi; F1: 0.33 &
\begin{minipage}[t]{\linewidth}
\begin{enumerate}
\setlength{\itemsep}{1pt}
\setlength{\parskip}{0pt}
\item ``... is a Serbian former professional American football punter in the National Football League for the Dallas Cowboys. He played college football at the University of South Dakota.''
\item ``... is an American former professional football player who was a defensive back for 13 seasons with the Dallas Cowboys ... He did not play college football at Utah State University ...''
\item ``... commentator for Cowboys games on Compass Media Networks' America's Team Radio Network since the 2011 season. He played college football at Arizona State University.''
\item ``Tubbs played three varsity years at the University of Oklahoma, and the Sooners won all 31 games in that period.''
\item ``James Gordon Miller ... was a punter in the National Football League ... Miller played college football for the University of Mississippi ...''
\end{enumerate}
\end{minipage} \\
\bottomrule
\end{tabularx}
\endgroup
\end{table*}

\begin{table*}[p]
\caption{2WikiMQA selected text excerpts.}
\label{tab:rag-case-2wikimqa}
\centering\scriptsize
\begingroup
\setlength{\tabcolsep}{3pt}
\renewcommand{\arraystretch}{1.08}
\begin{tabularx}{\textwidth}{@{}p{0.20\textwidth}X@{}}
\toprule
\multicolumn{2}{@{}p{\textwidth}@{}}{\textbf{Question.} What is the place of birth of Aleksey Greig's father? \quad \textbf{Gold answer.} Inverkeithing.} \\
\midrule
\textbf{Method} & \textbf{Chosen text excerpts} \\
\midrule
\textsc{GIP}$(\lambda=0.8)$; prediction: Inverkeithing; F1: 1.00 &
\begin{minipage}[t]{\linewidth}
\begin{enumerate}
\setlength{\itemsep}{1pt}
\setlength{\parskip}{0pt}
\item ``... Greig's grandfather Charles was an emigrant from Scotland. His father Samuil was an admiral in the Russian Imperial Navy.''
\item ``Dr Rogerson ... ordered him to proceed immediately to Revel ... The ceremonial of the admiral's funeral in the Tallinn Cathedral ...''
\item ``Vice-Admiral Samuel Greig ... (30 November 1735, Inverkeithing, Fife, Scotland - 26 October 1788, Tallinn, Estonia, Russian Empire) ...''
\item ``Aleksey Samuilovich Greig ... born into the noble Greig family, was an admiral of the Imperial Russian Navy.''
\item ``Family Samuel Greig married Sarah ... He was father to Alexey Greig, admiral of the Imperial Russian Navy ...''
\end{enumerate}
\end{minipage} \\
\midrule
Top-k; prediction: St. Petersburg; F1: 0.00 &
\begin{minipage}[t]{\linewidth}
\begin{enumerate}
\setlength{\itemsep}{1pt}
\setlength{\parskip}{0pt}
\item ``... Greig's grandfather Charles was an emigrant from Scotland. His father Samuil was an admiral in the Russian Imperial Navy.''
\item ``Dr Rogerson ... ordered him to proceed immediately to Revel ... The ceremonial of the admiral's funeral in the Tallinn Cathedral ...''
\item ``Family Samuel Greig married Sarah ... He was father to Alexey Greig, admiral of the Imperial Russian Navy ...''
\item ``In 1816 Greig became Commander of the Black Sea Fleet ... he served as Military Governor of Sevastopol and Nikolayev ...''
\item ``He succeeded Paul Dacre as editor of the Daily Mail in September 2018 ...''
\end{enumerate}
\end{minipage} \\
\bottomrule
\end{tabularx}
\endgroup
\end{table*}

\begin{table*}[p]
\caption{MuSiQue selected text excerpts.}
\label{tab:rag-case-musique}
\centering\scriptsize
\begingroup
\setlength{\tabcolsep}{3pt}
\renewcommand{\arraystretch}{1.08}
\begin{tabularx}{\textwidth}{@{}p{0.20\textwidth}X@{}}
\toprule
\multicolumn{2}{@{}p{\textwidth}@{}}{\textbf{Question.} Who was the first African American student at the university Robert Khayat was educated at? \quad \textbf{Gold answer.} James Meredith.} \\
\midrule
\textbf{Method} & \textbf{Chosen text excerpts} \\
\midrule
\textsc{GIP}$(\lambda=0.8)$; prediction: Robert Robinson Taylor; F1: 0.00 &
\begin{minipage}[t]{\linewidth}
\begin{enumerate}
\setlength{\itemsep}{1pt}
\setlength{\parskip}{0pt}
\item ``Robert Robinson Taylor ... Taylor was the first African-American student enrolled at the Massachusetts Institute of Technology (MIT) ...''
\item ``Robert Conrad Khayat ... was the 15th Chancellor of the University of Mississippi ... He attended Moss Point High School and the University of Mississippi.''
\item ``Khayat later became a lawyer and taught law at the University of Mississippi School of Law. He was appointed chancellor in 1995.''
\item ``Juanita Karpf wrote the book Performing Racial Uplift ... Death Hackley collapsed on stage while performing in San Diego in 1921 ...''
\item ``An African American, she promoted racial pride through her support and promotion of music education for African-Americans.''
\end{enumerate}
\end{minipage} \\
\midrule
DPP$(\alpha=2)$; prediction: James Meredith; F1: 1.00 &
\begin{minipage}[t]{\linewidth}
\begin{enumerate}
\setlength{\itemsep}{1pt}
\setlength{\parskip}{0pt}
\item ``Robert Robinson Taylor ... Taylor was the first African-American student enrolled at the Massachusetts Institute of Technology (MIT) ...''
\item ``Robert Conrad Khayat ... was the 15th Chancellor of the University of Mississippi ... He attended Moss Point High School and the University of Mississippi.''
\item ``Juanita Karpf wrote the book Performing Racial Uplift ... Death Hackley collapsed on stage while performing in San Diego in 1921 ...''
\item ``Meredith's admission is regarded as a pivotal moment in the history of civil rights in the United States. Many students harassed Meredith during his two semesters on campus ...''
\item ``... Alpha Phi Alpha fraternity. He died of a heart ailment at his home in Washington, D.C., February 28, 1998 ...''
\end{enumerate}
\end{minipage} \\
\bottomrule
\end{tabularx}
\endgroup
\end{table*}

\begin{table*}[p]
\caption{MultiFieldQA selected text excerpts.}
\label{tab:rag-case-multifieldqa}
\centering\scriptsize
\begingroup
\setlength{\tabcolsep}{3pt}
\renewcommand{\arraystretch}{1.08}
\begin{tabularx}{\textwidth}{@{}p{0.20\textwidth}X@{}}
\toprule
\multicolumn{2}{@{}p{\textwidth}@{}}{\textbf{Question.} When will BC leave Boston? \quad \textbf{Gold answer.} August 25.} \\
\midrule
\textbf{Method} & \textbf{Chosen text excerpts} \\
\midrule
\textsc{GIP}$(\lambda=0.8)$; prediction: Never mentioned; F1: 0.00 &
\begin{minipage}[t]{\linewidth}
\begin{enumerate}
\setlength{\itemsep}{1pt}
\setlength{\parskip}{0pt}
\item ``... one of the great under-appreciated aspects of BC sports is `For Boston.' ... BCeagles.com put out another player Q\&A ...''
\item ``... BCeagles.com released the bios of the six new freshmen and two transfers for the men's soccer team. Four of the newbies hail from New York.''
\item ``Boston is ideal because the local universities could already supply many of the venues.''
\item ``The closer we get to kickoff and the more I read, the more I think BC might surprise people this year.''
\item ``Canada will be new there this fall ... If Canada became a candidate for BC, I would have hesitation about his time at Indiana.''
\end{enumerate}
\end{minipage} \\
\midrule
xQuAD; prediction: August 25; F1: 1.00 &
\begin{minipage}[t]{\linewidth}
\begin{enumerate}
\setlength{\itemsep}{1pt}
\setlength{\parskip}{0pt}
\item ``Canada will be new there this fall ... If Canada became a candidate for BC, I would have hesitation about his time at Indiana.''
\item ``... BCeagles.com released the bios of the six new freshmen and two transfers for the men's soccer team. Four of the newbies hail from New York.''
\item ``... one of the great under-appreciated aspects of BC sports is `For Boston.' ... BCeagles.com put out another player Q\&A ...''
\item ``August 31 Barcelona - Game 3 ... September 1 Barcelona - Game 4 ... The team leaves Boston on August 25.''
\item ``Although the Eagle is now synonymous with BC ... Father McLaughlin ... suggested the eagle due to its majesty, power and freedom.''
\end{enumerate}
\end{minipage} \\
\bottomrule
\end{tabularx}
\endgroup
\end{table*}

These cases illustrate why we report final answer F1 in addition to retrieval-side metrics. Exact selected text can contain a correct bridge, an answer-bearing segment, or a semantically plausible distractor, and the generator outcome depends on which of these signals is present in the selected context.

%% file: appendix/significance_summary_table.tex
\begin{table*}[t]
\caption{Summary of significance tests for the supervised fine-tuning comparisons.}
\label{tab:significance-summary}
\centering\small
\begin{tabular}{lcccc}
\toprule
\textbf{Method} & \textbf{Sig. Wins} & \textbf{Sig. Losses} & \textbf{Neutral} & \textbf{Total} \\
\midrule
GIP+MA (All) & 34 (37.0\%) & 8 (8.7\%) & 50 (54.3\%) & 92 \\
GIP+SC (All) & 25 (27.2\%) & 5 (5.4\%) & 62 (67.4\%) & 92 \\
\midrule
GIP+MA vs Random & 6 (30.0\%) & 0 (0.0\%) & 14 (70.0\%) & 20 \\
GIP+SC vs Random & 4 (20.0\%) & 1 (5.0\%) & 15 (75.0\%) & 20 \\
\bottomrule
\end{tabular}
\end{table*}

%% file: generated/instruction_detail_table.tex

\begin{table*}[t]
\caption{Complete MT-Bench, BBH, and Alpaca-Eval-V2 performance with standard errors.}
\label{tab:mt-bench-bbh-detailed}
\centering\footnotesize
\begin{adjustbox}{width=\textwidth}
\begin{tabular}{lc|ccc|ccc|ccc}
\toprule
& & \multicolumn{3}{c|}{\textbf{Qwen3-4B}} & \multicolumn{3}{c|}{\textbf{Mistral-7B}} & \multicolumn{3}{c}{\textbf{Qwen3-8B}} \\
\textbf{Method} & \textbf{Data} & MT-Bench & BBH & A-Eval LC\% & MT-Bench & BBH & A-Eval LC\% & MT-Bench & BBH & A-Eval LC\% \\
\midrule
Full & 100\% (52K) & 6.69 $\pm$ 0.21 & 73.20 $\pm$ 0.55 & \textbf{4.88} $\pm$ 0.52 & 3.64 $\pm$ 0.20 & \textbf{58.40} $\pm$ 0.61 & 5.77 $\pm$ 0.58 & 7.74 $\pm$ 0.20 & 76.70 $\pm$ 0.53 & \underline{6.91} $\pm$ 0.66 \\
\midrule
Random & 2\% (1K) & 7.59 $\pm$ 0.19 & 69.40 $\pm$ 0.57 & 2.62 $\pm$ 0.58 & 3.66 $\pm$ 0.19 & 57.40 $\pm$ 0.61 & 4.71 $\pm$ 0.51 & 7.98 $\pm$ 0.19 & 72.30 $\pm$ 0.55 & 4.99 $\pm$ 0.72 \\
Alpaca-longest & 2\% (1K) & 7.45 $\pm$ 0.21 & 73.30 $\pm$ 0.55 & 2.48 $\pm$ 0.60 & \underline{3.89} $\pm$ 0.19 & 55.60 $\pm$ 0.62 & 2.69 $\pm$ 0.51 & \underline{8.20} $\pm$ 0.18 & 75.80 $\pm$ 0.53 & 4.24 $\pm$ 0.69 \\
Alpagasus & 2\% (1K) & 7.59 $\pm$ 0.20 & 63.60 $\pm$ 0.60 & 4.17 $\pm$ 0.72 & 3.71 $\pm$ 0.19 & 56.50 $\pm$ 0.61 & 4.81 $\pm$ 0.53 & 7.86 $\pm$ 0.19 & 71.90 $\pm$ 0.56 & 5.51 $\pm$ 0.74 \\
CaR & 2\% (1K) & \underline{7.68} $\pm$ 0.19 & 73.80 $\pm$ 0.54 & 3.26 $\pm$ 0.58 & 3.83 $\pm$ 0.20 & 57.40 $\pm$ 0.61 & 4.49 $\pm$ 0.54 & 8.16 $\pm$ 0.17 & 76.20 $\pm$ 0.53 & 3.66 $\pm$ 0.63 \\
\midrule
GIP+MA & 2\% (1K) & 7.67 $\pm$ 0.20 & \underline{74.30} $\pm$ 0.54 & 3.74 $\pm$ 0.62 & 3.83 $\pm$ 0.20 & 57.60 $\pm$ 0.61 & \underline{6.22} $\pm$ 0.53 & 7.99 $\pm$ 0.19 & \textbf{80.50} $\pm$ 0.52 & 5.64 $\pm$ 0.62 \\
GIP+MA & 1\% (512) & \textbf{7.78} $\pm$ 0.19 & 68.00 $\pm$ 0.58 & \underline{4.42} $\pm$ 0.72 & \textbf{3.99} $\pm$ 0.21 & 56.60 $\pm$ 0.61 & \textbf{6.59} $\pm$ 0.58 & \textbf{8.26} $\pm$ 0.16 & \underline{79.10} $\pm$ 0.52 & 4.85 $\pm$ 0.60 \\
GIP+SC & 2\% (1K) & 7.53 $\pm$ 0.20 & \textbf{74.40} $\pm$ 0.54 & 3.34 $\pm$ 0.51 & 3.71 $\pm$ 0.20 & \underline{57.70} $\pm$ 0.61 & 4.83 $\pm$ 0.50 & 8.04 $\pm$ 0.19 & \underline{79.10} $\pm$ 0.53 & \textbf{7.05} $\pm$ 0.69 \\
GIP+SC & 1\% (512) & 7.58 $\pm$ 0.20 & 73.90 $\pm$ 0.54 & 3.23 $\pm$ 0.60 & 3.71 $\pm$ 0.19 & 57.40 $\pm$ 0.61 & 4.85 $\pm$ 0.53 & 7.88 $\pm$ 0.20 & 76.50 $\pm$ 0.53 & 4.72 $\pm$ 0.65 \\
\bottomrule
\end{tabular}
\end{adjustbox}
\end{table*}

%% file: appendix/significance_tables.tex
\begin{table*}[t]
\caption{Significance test: GIP+MA at 2\% (1K) vs baselines. P-values are computed from aggregate means and standard errors, not paired per-instance outputs. * p$<$0.05, ** p$<$0.01, *** p$<$0.001.}
\label{tab:significance-mpma-2pct}
\centering
\footnotesize
\begin{adjustbox}{width=0.65\textwidth}
\begin{tabular}{l|cc|cc|cc}
\toprule
& \multicolumn{2}{c|}{\textbf{Mistral-7B}} & \multicolumn{2}{c|}{\textbf{Qwen3-8B}} & \multicolumn{2}{c}{\textbf{Qwen3-4B}} \\
\textbf{Baseline} & \textbf{MT-Bench} & \textbf{BBH} & \textbf{MT-Bench} & \textbf{BBH} & \textbf{MT-Bench} & \textbf{BBH} \\
\midrule
Full & 0.901 & 0.354 & $<$ 0.001*** & $<$ 0.001*** & 0.015* & 0.154 \\
Random & 0.732 & 0.817 & 1.000 & $<$ 0.001*** & 0.490 & $<$ 0.001*** \\
Alpaca-longest & 0.115 & 0.022* & 0.278 & $<$ 0.001*** & 0.327 & 0.195 \\
Alpagasus & 0.700 & 0.202 & 0.751 & $<$ 0.001*** & 0.952 & $<$ 0.001*** \\
CaR & 0.159 & 0.817 & 0.168 & $<$ 0.001*** & 0.002** & 0.513 \\
\bottomrule
\end{tabular}
\end{adjustbox}
\end{table*}

\begin{table*}[t]
\caption{Significance test: GIP+MA at 1\% (512) vs baselines. P-values are computed from aggregate means and standard errors, not paired per-instance outputs. * p$<$0.05, ** p$<$0.01, *** p$<$0.001.}
\label{tab:significance-mpma-1pct}
\centering
\footnotesize
\begin{adjustbox}{width=0.65\textwidth}
\begin{tabular}{l|cc|cc|cc}
\toprule
& \multicolumn{2}{c|}{\textbf{Mistral-7B}} & \multicolumn{2}{c|}{\textbf{Qwen3-8B}} & \multicolumn{2}{c}{\textbf{Qwen3-4B}} \\
\textbf{Baseline} & \textbf{MT-Bench} & \textbf{BBH} & \textbf{MT-Bench} & \textbf{BBH} & \textbf{MT-Bench} & \textbf{BBH} \\
\midrule
Full & 0.138 & 0.037* & $<$ 0.001*** & 0.001** & $<$ 0.001*** & $<$ 0.001*** \\
Random & 0.086 & 0.354 & 0.169 & $<$ 0.001*** & 0.878 & 0.085 \\
Alpaca-longest & 0.005** & 0.250 & 0.013* & $<$ 0.001*** & 0.900 & $<$ 0.001*** \\
Alpagasus & 0.079 & 0.908 & 0.099 & $<$ 0.001*** & 0.425 & $<$ 0.001*** \\
CaR & 0.008** & 0.354 & 0.005** & $<$ 0.001*** & $<$ 0.001*** & $<$ 0.001*** \\
\bottomrule
\end{tabular}
\end{adjustbox}
\end{table*}

\begin{table*}[t]
\caption{Significance test: GIP+SC at 2\% (1K) vs baselines. P-values are computed from aggregate means and standard errors, not paired per-instance outputs. * p$<$0.05, ** p$<$0.01, *** p$<$0.001.}
\label{tab:significance-mpsc-2pct}
\centering
\footnotesize
\begin{adjustbox}{width=0.65\textwidth}
\begin{tabular}{l|cc|cc|cc}
\toprule
& \multicolumn{2}{c|}{\textbf{Mistral-7B}} & \multicolumn{2}{c|}{\textbf{Qwen3-8B}} & \multicolumn{2}{c}{\textbf{Qwen3-4B}} \\
\textbf{Baseline} & \textbf{MT-Bench} & \textbf{BBH} & \textbf{MT-Bench} & \textbf{BBH} & \textbf{MT-Bench} & \textbf{BBH} \\
\midrule
Full & 0.629 & 0.417 & 0.008** & 0.001** & 0.169 & 0.120 \\
Random & 0.771 & 0.728 & 0.026* & $<$ 0.001*** & 0.072 & $<$ 0.001*** \\
Alpaca-longest & 0.362 & 0.016* & 0.265 & $<$ 0.001*** & 0.034* & 0.154 \\
Alpagasus & 0.803 & 0.164 & 0.067 & $<$ 0.001*** & 0.244 & $<$ 0.001*** \\
CaR & 0.455 & 0.728 & 0.411 & $<$ 0.001*** & 0.040* & 0.432 \\
\bottomrule
\end{tabular}
\end{adjustbox}
\end{table*}

\begin{table*}[t]
\caption{Significance test: GIP+SC at 1\% (512) vs baselines. P-values are computed from aggregate means and standard errors, not paired per-instance outputs. * p$<$0.05, ** p$<$0.01, *** p$<$0.001.}
\label{tab:significance-mpsc-1pct}
\centering
\footnotesize
\begin{adjustbox}{width=0.65\textwidth}
\begin{tabular}{l|cc|cc|cc}
\toprule
& \multicolumn{2}{c|}{\textbf{Mistral-7B}} & \multicolumn{2}{c|}{\textbf{Qwen3-8B}} & \multicolumn{2}{c}{\textbf{Qwen3-4B}} \\
\textbf{Baseline} & \textbf{MT-Bench} & \textbf{BBH} & \textbf{MT-Bench} & \textbf{BBH} & \textbf{MT-Bench} & \textbf{BBH} \\
\midrule
Full & 0.328 & 0.246 & $<$ 0.001*** & 0.790 & 0.011* & 0.364 \\
Random & 0.227 & 1.000 & 0.975 & $<$ 0.001*** & 0.519 & $<$ 0.001*** \\
Alpaca-longest & 0.018* & 0.039* & 0.244 & 0.350 & 0.347 & 0.436 \\
Alpagasus & 0.212 & 0.297 & 0.719 & $<$ 0.001*** & 1.000 & $<$ 0.001*** \\
CaR & 0.028* & 1.000 & 0.142 & 0.689 & 0.001** & 0.896 \\
\bottomrule
\end{tabular}
\end{adjustbox}
\end{table*}

%% file: generated/gsm_detail_table.tex

\begin{table*}[t]
\caption{Complete GSM8K mathematical reasoning performance with standard errors.}
\label{tab:GSM8K-detailed}
\centering\footnotesize
\begin{adjustbox}{width=\textwidth}
\begin{tabular}{l|cccc|cccc|cccc}
\toprule
& \multicolumn{4}{c|}{\textbf{Mistral-7B}} & \multicolumn{4}{c|}{\textbf{Qwen3-8B}} & \multicolumn{4}{c}{\textbf{Qwen3-32B}} \\
\textbf{Method} & 2.5\% & 5\% & 10\% & 20\% & 2.5\% & 5\% & 10\% & 20\% & 2.5\% & 5\% & 10\% & 20\% \\
\midrule
Full & \multicolumn{4}{c|}{52.46 $\pm$ 1.38} & \multicolumn{4}{c|}{81.96 $\pm$ 1.06} & \multicolumn{4}{c}{87.64 $\pm$ 0.91} \\
\midrule
Random & 37.76 $\pm$ 1.33 & 43.52 $\pm$ 1.37 & 45.72 $\pm$ 1.37 & 46.74 $\pm$ 1.37 & 78.32 $\pm$ 1.13 & 80.21 $\pm$ 1.10 & 79.53 $\pm$ 1.11 & 81.65 $\pm$ 1.07 & 78.24 $\pm$ 1.14 & 87.14 $\pm$ 0.92 & 85.67 $\pm$ 0.96 & 87.29 $\pm$ 0.92 \\
DiSF & 39.88 $\pm$ 1.35 & 41.55 $\pm$ 1.36 & 41.55 $\pm$ 1.36 & 42.46 $\pm$ 1.36 & 75.06 $\pm$ 1.19 & 79.98 $\pm$ 1.10 & 79.98 $\pm$ 1.10 & 80.67 $\pm$ 1.09 & 81.65 $\pm$ 1.07 & 87.04 $\pm$ 0.92 & 87.04 $\pm$ 0.92 & 88.40 $\pm$ 0.88 \\
DSIR & 37.07 $\pm$ 1.33 & 42.15 $\pm$ 1.36 & 42.61 $\pm$ 1.36 & 42.08 $\pm$ 1.36 & 80.74 $\pm$ 1.09 & 81.50 $\pm$ 1.07 & 81.96 $\pm$ 1.06 & 83.24 $\pm$ 1.03 & 80.74 $\pm$ 1.09 & 89.16 $\pm$ 0.86 & 87.79 $\pm$ 0.90 & 88.70 $\pm$ 0.87 \\
LESS & 31.84 $\pm$ 1.28 & 33.13 $\pm$ 1.30 & 42.91 $\pm$ 1.36 & 47.69 $\pm$ 1.38 & 79.76 $\pm$ 1.11 & 79.45 $\pm$ 1.11 & 80.29 $\pm$ 1.10 & 79.45 $\pm$ 1.11 & 69.98 $\pm$ 1.26 & 87.26 $\pm$ 0.92 & 87.95 $\pm$ 0.90 & 87.41 $\pm$ 0.91 \\
\midrule
GIP+MA & 42.99 $\pm$ 1.36 & 45.64 $\pm$ 1.37 & 45.94 $\pm$ 1.37 & 47.46 $\pm$ 1.37 & 81.58 $\pm$ 1.07 & 81.05 $\pm$ 1.08 & 83.09 $\pm$ 1.03 & 83.24 $\pm$ 1.03 & 84.31 $\pm$ 1.00 & 87.57 $\pm$ 0.91 & 88.25 $\pm$ 0.89 & 88.78 $\pm$ 0.87 \\
GIP+SC & 38.89 $\pm$ 1.34 & 43.29 $\pm$ 1.36 & 46.63 $\pm$ 1.37 & 49.81 $\pm$ 1.38 & 80.36 $\pm$ 1.09 & 80.21 $\pm$ 1.10 & 81.65 $\pm$ 1.07 & 82.26 $\pm$ 1.05 & 84.15 $\pm$ 1.01 & 87.04 $\pm$ 0.92 & 87.34 $\pm$ 0.92 & 87.56 $\pm$ 0.91 \\
\bottomrule
\end{tabular}
\end{adjustbox}
\end{table*}

%% file: generated/fishersft_detail_table.tex

\begin{table}[H]
\caption{Matched GSM8K accuracy for the settings in which we evaluated FisherSFT. Best results within each model--budget column are in \textbf{bold}; uncertainty is the binomial standard error over 1{,}319 test questions.}
\label{tab:fishersft-comparison}
\centering\footnotesize
\setlength{\tabcolsep}{3.2pt}
\begin{tabular}{llcc}
\toprule
\textbf{Model} & \textbf{Method} & \textbf{2.5\%} & \textbf{20\%} \\
\midrule
Mistral-7B & FisherSFT & 33.89 $\pm$ 1.30 & 46.32 $\pm$ 1.37 \\
 & GIP+MA & \textbf{42.99} $\pm$ 1.36 & 47.46 $\pm$ 1.37 \\
 & GIP+SC & 38.89 $\pm$ 1.34 & \textbf{49.81} $\pm$ 1.38 \\
\midrule
Qwen3-8B & FisherSFT & 76.65 $\pm$ 1.16 & 81.27 $\pm$ 1.07 \\
 & GIP+MA & \textbf{81.58} $\pm$ 1.07 & \textbf{83.24} $\pm$ 1.03 \\
 & GIP+SC & 80.36 $\pm$ 1.09 & 82.26 $\pm$ 1.05 \\
\bottomrule
\end{tabular}
\end{table}

%% file: appendix/gsm8k_significance_tables.tex
\begin{table*}[t]
\caption{Significance test for GSM8K: GIP+MA vs baselines at each data percentage. P-values are computed from aggregate means and standard errors, not paired per-instance outputs. * p$<$0.05, ** p$<$0.01, *** p$<$0.001.}
\label{tab:gsm8k-significance-mpma}
\centering
\footnotesize
\begin{adjustbox}{width=0.65\textwidth}
\begin{tabular}{l|cccc|cccc}
\toprule
& \multicolumn{4}{c|}{\textbf{Qwen3-8B}} & \multicolumn{4}{c}{\textbf{Mistral-7B}} \\
\textbf{Baseline} & \textbf{2.5\%} & \textbf{5\%} & \textbf{10\%} & \textbf{20\%} & \textbf{2.5\%} & \textbf{5\%} & \textbf{10\%} & \textbf{20\%} \\
\midrule
Full & 0.801 & 0.548 & 0.445 & 0.387 & $<$ 0.001*** & $<$ 0.001*** & $<$ 0.001*** & 0.011* \\
Random & 0.037* & 0.586 & 0.019* & 0.285 & 0.006** & 0.272 & 0.910 & 0.712 \\
DiSF & $<$ 0.001*** & 0.488 & 0.039* & 0.087 & 0.105 & 0.034* & 0.023* & 0.010** \\
DSIR & 0.582 & 0.767 & 0.445 & 1.000 & 0.002** & 0.071 & 0.085 & 0.006** \\
\bottomrule
\end{tabular}
\end{adjustbox}
\end{table*}

\begin{table*}[t]
\caption{Significance test for GSM8K: GIP+SC vs baselines at each data percentage. P-values are computed from aggregate means and standard errors, not paired per-instance outputs. * p$<$0.05, ** p$<$0.01, *** p$<$0.001.}
\label{tab:gsm8k-significance-mpsc}
\centering
\footnotesize
\begin{adjustbox}{width=0.65\textwidth}
\begin{tabular}{l|cccc|cccc}
\toprule
& \multicolumn{4}{c|}{\textbf{Qwen3-8B}} & \multicolumn{4}{c}{\textbf{Mistral-7B}} \\
\textbf{Baseline} & \textbf{2.5\%} & \textbf{5\%} & \textbf{10\%} & \textbf{20\%} & \textbf{2.5\%} & \textbf{5\%} & \textbf{10\%} & \textbf{20\%} \\
\midrule
Full & 0.295 & 0.252 & 0.837 & 0.841 & $<$ 0.001*** & $<$ 0.001*** & 0.003** & 0.175 \\
Random & 0.198 & 1.000 & 0.169 & 0.684 & 0.551 & 0.905 & 0.640 & 0.116 \\
DiSF & 0.001** & 0.882 & 0.277 & 0.294 & 0.603 & 0.366 & 0.009** & $<$ 0.001*** \\
DSIR & 0.806 & 0.401 & 0.837 & 0.505 & 0.335 & 0.553 & 0.038* & $<$ 0.001*** \\
\bottomrule
\end{tabular}
\end{adjustbox}
\end{table*}

%% file: iclr2026_conference.bib
@article{chen2018fast,
  title={Fast Greedy {MAP} Inference for Determinantal Point Process to Improve Recommendation Diversity},
  author={Chen, Laming and Zhang, Guoxin and Zhou, Eric},
  journal={Advances in Neural Information Processing Systems},
  volume={31},
  year={2018}
}

@inproceedings{bukharinDataDiversityMatters2024,
  title = {Data Diversity Matters for Robust Instruction Tuning},
  author = {Bukharin, Alexander and Li, Shiyang and Wang, Zhengyang and Yang, Jingfeng and Yin, Bing and Li, Xian and Zhang, Chao and Zhao, Tuo and Jiang, Haoming},
  booktitle = {Findings of the Association for Computational Linguistics: EMNLP 2024},
  pages = {3411--3425},
  year = {2024},
  publisher = {Association for Computational Linguistics},
  doi = {10.18653/v1/2024.findings-emnlp.195},
  url = {https://aclanthology.org/2024.findings-emnlp.195/}
}

@inproceedings{chenDataJuicerOneStopData2023,
  title = {Data-{Juicer}: A One-Stop Data Processing System for Large Language Models},
  author = {Chen, Daoyuan and Huang, Yilun and Ma, Zhijian and Chen, Hesen and Pan, Xuchen and Ge, Ce and Gao, Dawei and Xie, Yuexiang and Liu, Zhaoyang and Gao, Jinyang and Li, Yaliang and Ding, Bolin and Zhou, Jingren},
  booktitle = {Companion of the 2024 International Conference on Management of Data},
  pages = {120--134},
  year = {2024},
  publisher = {Association for Computing Machinery},
  doi = {10.1145/3626246.3653385},
  url = {https://doi.org/10.1145/3626246.3653385}
}

@misc{chenDiversitySyntheticData2024,
  title = {On the {{Diversity}} of {{Synthetic Data}} and Its {{Impact}} on {{Training Large Language Models}}},
  author = {Chen, Hao and Waheed, Abdul and Li, Xiang and Wang, Yidong and Wang, Jindong and Raj, Bhiksha and Abdin, Marah I.},
  year = {2024},
  month = oct,
  number = {arXiv:2410.15226},
  eprint = {2410.15226},
  publisher = {arXiv},
  urldate = {2024-11-29},
  archiveprefix = {arXiv}
}

@misc{duMoDSModelorientedData2023,
  title = {{{MoDS}}: {{Model-oriented Data Selection}} for {{Instruction Tuning}}},
  shorttitle = {{{MoDS}}},
  author = {Du, Qianlong and Zong, Chengqing and Zhang, Jiajun},
  year = {2023},
  month = nov,
  number = {arXiv:2311.15653},
  eprint = {2311.15653},
  primaryclass = {cs},
  publisher = {arXiv},
  urldate = {2023-12-09},
  archiveprefix = {arXiv}
}

@inproceedings{leeDeduplicatingTrainingData2022,
  title = {Deduplicating Training Data Makes Language Models Better},
  author = {Lee, Katherine and Ippolito, Daphne and Nystrom, Andrew and Zhang, Chiyuan and Eck, Douglas and {Callison-Burch}, Chris and Carlini, Nicholas},
  booktitle = {Proceedings of the 60th Annual Meeting of the Association for Computational Linguistics (Volume 1: Long Papers)},
  pages = {8424--8445},
  year = {2022},
  publisher = {Association for Computational Linguistics},
  doi = {10.18653/v1/2022.acl-long.577},
  url = {https://aclanthology.org/2022.acl-long.577/}
}

@inproceedings{liTunaInstructionTuning2023,
  title = {{Tuna}: Instruction Tuning Using Feedback from Large Language Models},
  author = {Li, Haoran and Liu, Yiran and Zhang, Xingxing and Lu, Wei and Wei, Furu},
  booktitle = {Findings of the Association for Computational Linguistics: EMNLP 2023},
  pages = {15146--15163},
  year = {2023},
  publisher = {Association for Computational Linguistics},
  doi = {10.18653/v1/2023.findings-emnlp.1011},
  url = {https://aclanthology.org/2023.findings-emnlp.1011/}
}

@inproceedings{senerActiveLearningConvolutional2018,
  title = {Active Learning for Convolutional Neural Networks: A Core-Set Approach},
  author = {Sener, Ozan and Savarese, Silvio},
  booktitle = {International Conference on Learning Representations},
  year = {2018},
  url = {https://openreview.net/forum?id=H1aIuk-RW}
}

@misc{taori2023alpaca,
  title={Alpaca: A Strong, Replicable Instruction-Following Model},
  author={Taori, Rohan and Gulrajani, Ishaan and Zhang, Tianyi and Dubois, Yann and Li, Xuechen and Guestrin, Carlos and Liang, Percy and Hashimoto, Tatsunori B},
  year={2023},
  howpublished={Stanford Center for Research on Foundation Models},
  url={https://crfm.stanford.edu/2023/03/13/alpaca.html}
}

@article{knops2006normalized,
  title={Normalized mutual information based registration using k-means clustering and shading correction},
  author={Knops, Zeger F and Maintz, JB Antoine and Viergever, Max A and Pluim, Josien PW},
  journal={Medical image analysis},
  volume={10},
  number={3},
  pages={432--439},
  year={2006},
  publisher={Elsevier},
  doi={10.1016/j.media.2005.03.009}
}

@article{he2008k,
  title={{k-ANMI}: A mutual information based clustering algorithm for categorical data},
  author={He, Zengyou and Xu, Xiaofei and Deng, Shengchun},
  journal={Information Fusion},
  volume={9},
  number={2},
  pages={223--233},
  year={2008},
  publisher={Elsevier},
  doi={10.1016/j.inffus.2006.05.006}
}

@article{kraskov2005hierarchical,
  title={Hierarchical clustering using mutual information},
  author={Kraskov, Alexander and St{\"o}gbauer, Harald and Andrzejak, Ralph G and Grassberger, Peter},
  journal={Europhysics Letters},
  volume={70},
  number={2},
  pages={278--284},
  year={2005},
  publisher={IOP Publishing},
  doi={10.1209/epl/i2004-10483-y}
}

@inproceedings{do2021clustering,
  title={Clustering by maximizing mutual information across views},
  author={Do, Kien and Tran, Truyen and Venkatesh, Svetha},
  booktitle={Proceedings of the IEEE/CVF international conference on computer vision},
  pages={9928--9938},
  year={2021}
}

@article{kraskov2004estimating,
  title={Estimating mutual information},
  author={Kraskov, Alexander and St{\"o}gbauer, Harald and Grassberger, Peter},
  journal={Physical Review E},
  volume={69},
  number={6},
  pages={066138},
  year={2004},
  publisher={American Physical Society},
  doi={10.1103/PhysRevE.69.066138}
}

@article{calandriello2014semi,
  title={Semi-supervised information-maximization clustering},
  author={Calandriello, Daniele and Niu, Gang and Sugiyama, Masashi},
  journal={Neural networks},
  volume={57},
  pages={103--111},
  year={2014},
  publisher={Elsevier},
  doi={10.1016/j.neunet.2014.05.016}
}

@article{sugiyama2014information,
  title={Information-maximization clustering based on squared-loss mutual information},
  author={Sugiyama, Masashi and Niu, Gang and Yamada, Makoto and Kimura, Manabu and Hachiya, Hirotaka},
  journal={Neural Computation},
  volume={26},
  number={1},
  pages={84--131},
  year={2014},
  publisher={MIT Press},
  doi={10.1162/NECO_a_00534}
}

@inproceedings{ge2024clustering,
  title={Clustering and Ranking: Diversity-Preserved Instruction Selection through Expert-Aligned Quality Estimation},
  author={Ge, Yuan and Liu, Yilun and Hu, Chi and Meng, Weibin and Tao, Shimin and Zhao, Xiaofeng and Ma, Hongxia and Zhang, Li and Chen, Boxing and Yang, Hao and Li, Bei and Xiao, Tong and Zhu, Jingbo},
  booktitle={Proceedings of the 2024 Conference on Empirical Methods in Natural Language Processing},
  pages={464--478},
  year={2024},
  publisher={Association for Computational Linguistics},
  doi={10.18653/v1/2024.emnlp-main.28},
  url={https://aclanthology.org/2024.emnlp-main.28/}
}

@inproceedings{zhao2024long,
  title={Long is more for alignment: A simple but tough-to-beat baseline for instruction fine-tuning},
  author={Zhao, Hao and Andriushchenko, Maksym and Croce, Francesco and Flammarion, Nicolas},
  booktitle={Proceedings of the 41st International Conference on Machine Learning},
  volume={235},
  series={Proceedings of Machine Learning Research},
  pages={60674--60703},
  year={2024},
  publisher={PMLR},
  url={https://proceedings.mlr.press/v235/zhao24b.html}
}

@inproceedings{chen2023alpagasus,
  title={{AlpaGasus}: Training a Better Alpaca with Fewer Data},
  author={Chen, Lichang and Li, Shiyang and Yan, Jun and Wang, Hai and Gunaratna, Kalpa and Yadav, Vikas and Tang, Zheng and Srinivasan, Vijay and Zhou, Tianyi and Huang, Heng and Jin, Hongxia},
  booktitle={The Twelfth International Conference on Learning Representations},
  year={2024},
  url={https://openreview.net/forum?id=FdVXgSJhvz}
}

@article{zheng2023judging,
  title={Judging {LLM-as-a-Judge} with {MT-Bench} and {Chatbot Arena}},
  author={Zheng, Lianmin and Chiang, Wei-Lin and Sheng, Ying and Zhuang, Siyuan and Wu, Zhanghao and Zhuang, Yonghao and Lin, Zi and Li, Zhuohan and Li, Dacheng and Xing, Eric P. and Zhang, Hao and Gonzalez, Joseph E. and Stoica, Ion},
  journal={Advances in Neural Information Processing Systems},
  volume={36},
  pages={46595--46623},
  year={2023},
  url={https://proceedings.neurips.cc/paper_files/paper/2023/hash/91f18a1287b398d378ef22505bf41832-Abstract-Datasets_and_Benchmarks.html}
}

@inproceedings{warner2024smarter,
  title={Smarter, Better, Faster, Longer: A Modern Bidirectional Encoder for Fast, Memory Efficient, and Long Context Finetuning and Inference},
  author={Warner, Benjamin and Chaffin, Antoine and Clavi{\'e}, Benjamin and Weller, Orion and Hallstr{\"o}m, Oskar and Taghadouini, Said and Gallagher, Alexis and Biswas, Raja and Ladhak, Faisal and Aarsen, Tom and Adams, Griffin Thomas and Howard, Jeremy and Poli, Iacopo},
  booktitle={Proceedings of the 63rd Annual Meeting of the Association for Computational Linguistics (Volume 1: Long Papers)},
  pages={2526--2547},
  year={2025},
  publisher={Association for Computational Linguistics},
  doi={10.18653/v1/2025.acl-long.127},
  url={https://aclanthology.org/2025.acl-long.127/}
}

@article{achiam2023gpt,
  title={{GPT-4} Technical Report},
  author={{OpenAI}},
  journal={arXiv preprint arXiv:2303.08774},
  year={2023},
  url={https://arxiv.org/abs/2303.08774}
}

@article{qwen2025qwen25technicalreport,
  title={{Qwen2.5} Technical Report},
  author={An Yang and Baosong Yang and Beichen Zhang and Binyuan Hui and Bo Zheng and Bowen Yu and Chengyuan Li and Dayiheng Liu and Fei Huang and Haoran Wei and Huan Lin and Jian Yang and Jianhong Tu and Jianwei Zhang and Jianxin Yang and Jiaxi Yang and Jingren Zhou and Junyang Lin and Kai Dang and Keming Lu and Keqin Bao and Kexin Yang and Le Yu and Mei Li and Mingfeng Xue and Pei Zhang and Qin Zhu and Rui Men and Runji Lin and Tianhao Li and Tingyu Xia and Xingzhang Ren and Xuancheng Ren and Yang Fan and Yang Su and Yichang Zhang and Yu Wan and Yuqiong Liu and Zeyu Cui and Zhenru Zhang and Zihan Qiu},
  journal={arXiv preprint arXiv:2412.15115},
  year={2024},
  url={https://arxiv.org/abs/2412.15115}
}

@misc{llamaTeam2024herd,
  title={The {Llama 3} Herd of Models},
  author={{Llama Team, AI at Meta}},
  year={2024},
  eprint={2407.21783},
  archivePrefix={arXiv},
  primaryClass={cs.AI},
  url={https://arxiv.org/abs/2407.21783}
}

@misc{hoffmann2022trainingcomputeoptimallargelanguage,
      title={Training Compute-Optimal Large Language Models}, 
      author={Jordan Hoffmann and Sebastian Borgeaud and Arthur Mensch and Elena Buchatskaya and Trevor Cai and Eliza Rutherford and Diego de Las Casas and Lisa Anne Hendricks and Johannes Welbl and Aidan Clark and Tom Hennigan and Eric Noland and Katie Millican and George van den Driessche and Bogdan Damoc and Aurelia Guy and Simon Osindero and Karen Simonyan and Erich Elsen and Jack W. Rae and Oriol Vinyals and Laurent Sifre},
      year={2022},
      eprint={2203.15556},
      archivePrefix={arXiv},
      primaryClass={cs.CL},
      url={https://arxiv.org/abs/2203.15556}, 
}

@inproceedings{wang2023selfinstructaligninglanguagemodels,
  title={{Self-Instruct}: Aligning Language Models with Self-Generated Instructions},
  author={Wang, Yizhong and Kordi, Yeganeh and Mishra, Swaroop and Liu, Alisa and Smith, Noah A. and Khashabi, Daniel and Hajishirzi, Hannaneh},
  booktitle={Proceedings of the 61st Annual Meeting of the Association for Computational Linguistics (Volume 1: Long Papers)},
  pages={13484--13508},
  year={2023},
  publisher={Association for Computational Linguistics},
  doi={10.18653/v1/2023.acl-long.754},
  url={https://aclanthology.org/2023.acl-long.754/}
}

@article{tropp2004greed,
  title={Greed is good: Algorithmic results for sparse approximation},
  author={Tropp, Joel A},
  journal={IEEE Transactions on Information theory},
  volume={50},
  number={10},
  pages={2231--2242},
  year={2004},
  publisher={IEEE},
  doi={10.1109/TIT.2004.834793}
}

@phdthesis{lee2022new,
  title={New Information Inequalities with Applications to Statistics},
  author={Lee, Kuan-Yun},
  year={2022},
  school={University of California, Berkeley},
  url={https://escholarship.org/uc/item/083030s1}
}

@inproceedings{lee2021minimax,
  title={Minimax Bounds for Generalized Pairwise Comparisons},
  author={Lee, Kuan-Yun and Courtade, Thomas A},
  booktitle={2021 International Conference on Machine Learning (ICML) Workshop on Information-Theoretic Methods for Rigorous, Responsible, and Reliable Machine Learning},
  year={2021}
}

@inproceedings{lee2020linear,
  title={Linear models are most favorable among generalized linear models},
  author={Lee, Kuan-Yun and Courtade, Thomas A},
  booktitle={2020 IEEE International Symposium on Information Theory (ISIT)},
  pages={1213--1218},
  year={2020},
  organization={IEEE},
  doi={10.1109/ISIT44484.2020.9174124}
}

@article{tropp2007signal,
  title={Signal recovery from random measurements via orthogonal matching pursuit},
  author={Tropp, Joel A and Gilbert, Anna C},
  journal={IEEE Transactions on information theory},
  volume={53},
  number={12},
  pages={4655--4666},
  year={2007},
  publisher={IEEE},
  doi={10.1109/TIT.2007.909108}
}

@inproceedings{pati1993orthogonal,
  title={Orthogonal matching pursuit: Recursive function approximation with applications to wavelet decomposition},
  author={Pati, Yagyensh Chandra and Rezaiifar, Ramin and Krishnaprasad, Perinkulam Sambamurthy},
  booktitle={Proceedings of 27th Asilomar conference on signals, systems and computers},
  pages={40--44},
  year={1993},
  organization={IEEE},
  doi={10.1109/ACSSC.1993.342465}
}

@misc{kaplan2020scalinglawsneurallanguage,
      title={Scaling Laws for Neural Language Models}, 
      author={Jared Kaplan and Sam McCandlish and Tom Henighan and Tom B. Brown and Benjamin Chess and Rewon Child and Scott Gray and Alec Radford and Jeffrey Wu and Dario Amodei},
      year={2020},
      eprint={2001.08361},
      archivePrefix={arXiv},
      primaryClass={cs.LG},
      url={https://arxiv.org/abs/2001.08361}, 
}

@article{mallat1993matching,
  title={Matching pursuits with time-frequency dictionaries},
  author={Mallat, St{\'e}phane G and Zhang, Zhifeng},
  journal={IEEE Transactions on signal processing},
  volume={41},
  number={12},
  pages={3397--3415},
  year={1993},
  publisher={IEEE},
  doi={10.1109/78.258082}
}

@inproceedings{hu2021loralowrankadaptationlarge,
  title={{LoRA}: Low-Rank Adaptation of Large Language Models},
  author={Hu, Edward J. and Shen, Yelong and Wallis, Phillip and Allen-Zhu, Zeyuan and Li, Yuanzhi and Wang, Shean and Wang, Lu and Chen, Weizhu},
  booktitle={International Conference on Learning Representations},
  year={2022},
  url={https://openreview.net/forum?id=nZeVKeeFYf9}
}

@misc{jiang2023mistral7b,
      title={{Mistral 7B}},
      author={Albert Q. Jiang and Alexandre Sablayrolles and Arthur Mensch and Chris Bamford and Devendra Singh Chaplot and Diego de las Casas and Florian Bressand and Gianna Lengyel and Guillaume Lample and Lucile Saulnier and Lélio Renard Lavaud and Marie-Anne Lachaux and Pierre Stock and Teven Le Scao and Thibaut Lavril and Thomas Wang and Timothée Lacroix and William El Sayed},
      year={2023},
      eprint={2310.06825},
      archivePrefix={arXiv},
      primaryClass={cs.CL},
      url={https://arxiv.org/abs/2310.06825}, 
}

@article{yang2025qwen3,
  title={{Qwen3} Technical Report},
  author={An Yang and Anfeng Li and Baosong Yang and Beichen Zhang and Binyuan Hui and Bo Zheng and Bowen Yu and Chang Gao and Chengen Huang and Chenxu Lv and Chujie Zheng and Dayiheng Liu and Fan Zhou and Fei Huang and Feng Hu and Hao Ge and Haoran Wei and Huan Lin and Jialong Tang and Jian Yang and Jianhong Tu and Jianwei Zhang and Jianxin Yang and Jiaxi Yang and Jing Zhou and Jingren Zhou and Junyang Lin and Kai Dang and Keqin Bao and Kexin Yang and Le Yu and Lianghao Deng and Mei Li and Mingfeng Xue and Mingze Li and Pei Zhang and Peng Wang and Qin Zhu and Rui Men and Ruize Gao and Shixuan Liu and Shuang Luo and Tianhao Li and Tianyi Tang and Wenbiao Yin and Xingzhang Ren and Xinyu Wang and Xinyu Zhang and Xuancheng Ren and Yang Fan and Yang Su and Yichang Zhang and Yinger Zhang and Yu Wan and Yuqiong Liu and Zekun Wang and Zeyu Cui and Zhenru Zhang and Zhipeng Zhou and Zihan Qiu},
  journal={arXiv preprint arXiv:2505.09388},
  year={2025},
  url={https://arxiv.org/abs/2505.09388}
}

@article{cobbe2021training,
  title={Training Verifiers to Solve Math Word Problems},
  author={Cobbe, Karl and Kosaraju, Vineet and Bavarian, Mohammad and Chen, Mark and Jun, Heewoo and Kaiser, Lukasz and Plappert, Matthias and Tworek, Jerry and Hilton, Jacob and Nakano, Reiichiro and Hesse, Christopher and Schulman, John},
  journal={arXiv preprint arXiv:2110.14168},
  year={2021},
  url={https://arxiv.org/abs/2110.14168}
}

@techreport{Settles2009ActiveLearningSurvey,
  author      = {Burr Settles},
  title       = {Active Learning Literature Survey},
  institution = {University of Wisconsin--Madison},
  number      = {CS-TR-1648},
  year        = {2009},
  url         = {https://burrsettles.com/pub/settles.activelearning.pdf}
}

@inproceedings{Xia2024LESS,
  author  = {Mengzhou Xia and Sadhika Malladi and Suchin Gururangan and Sanjeev Arora and Danqi Chen},
  title   = {{LESS}: Selecting Influential Data for Targeted Instruction Tuning},
  booktitle = {Proceedings of the 41st International Conference on Machine Learning},
  volume = {235},
  series = {Proceedings of Machine Learning Research},
  pages = {54104--54132},
  year    = {2024},
  publisher = {PMLR},
  url     = {https://proceedings.mlr.press/v235/xia24c.html}
}

@inproceedings{Xie2023DSIR,
  author    = {Sang Michael Xie and Shibani Santurkar and Tengyu Ma and Percy Liang},
  title     = {Data Selection for Language Models via Importance Resampling},
  booktitle = {Advances in Neural Information Processing Systems},
  volume    = {36},
  year      = {2023},
  doi       = {10.52202/075280-1482},
  url       = {https://proceedings.neurips.cc/paper_files/paper/2023/hash/6b9aa8f418bde2840d5f4ab7a02f663b-Abstract-Conference.html}
}

@inproceedings{Suzgun2022BBH,
  author  = {Mirac Suzgun and Nathan Scales and Nathanael Sch{\"a}rli and Sebastian Gehrmann and Yi Tay and Hyung Won Chung and Aakanksha Chowdhery and Quoc V. Le and Ed H. Chi and Denny Zhou and Jason Wei},
  title   = {Challenging {BIG-Bench} Tasks and Whether {Chain-of-Thought} Can Solve Them},
  booktitle = {Findings of the Association for Computational Linguistics: ACL 2023},
  pages = {13003--13051},
  year    = {2023},
  publisher = {Association for Computational Linguistics},
  doi = {10.18653/v1/2023.findings-acl.824},
  url     = {https://aclanthology.org/2023.findings-acl.824/}
}

@misc{bachem2017practical,
  title={Practical Coreset Constructions for Machine Learning},
  author={Bachem, Olivier and Lucic, Mario and Krause, Andreas},
  year={2017},
  eprint={1703.06476},
  archivePrefix={arXiv},
  primaryClass={stat.ML},
  doi={10.48550/arXiv.1703.06476},
  url={https://arxiv.org/abs/1703.06476}
}

@inproceedings{kwon2023efficient,
  title={Efficient Memory Management for Large Language Model Serving with {PagedAttention}},
  author={Woosuk Kwon and Zhuohan Li and Siyuan Zhuang and Ying Sheng and Lianmin Zheng and Cody Hao Yu and Joseph E. Gonzalez and Hao Zhang and Ion Stoica},
  booktitle={Proceedings of the 29th Symposium on Operating Systems Principles},
  pages={611--626},
  year={2023},
  publisher={Association for Computing Machinery},
  doi={10.1145/3600006.3613165}
}

@inproceedings{fan2025combattingdimensionalcollapsellm,
  title={Combatting Dimensional Collapse in {LLM} Pre-Training Data via Submodular File Selection},
  author={Ziqing Fan and Siyuan Du and Shengchao Hu and Pingjie Wang and Li Shen and Ya Zhang and Dacheng Tao and Yanfeng Wang},
  booktitle={The Thirteenth International Conference on Learning Representations},
  year={2025},
  url={https://openreview.net/forum?id=f4gF6AIHRy}
}

@misc{eval-harness,
  author       = {Gao, Leo and Tow, Jonathan and Abbasi, Baber and Biderman, Stella and Black, Sid and DiPofi, Anthony and Foster, Charles and Golding, Laurence and Hsu, Jeffrey and Le Noac'h, Alain and Li, Haonan and McDonell, Kyle and Muennighoff, Niklas and Ociepa, Chris and Phang, Jason and Reynolds, Laria and Schoelkopf, Hailey and Skowron, Aviya and Sutawika, Lintang and Tang, Eric and Thite, Anish and Wang, Ben and Wang, Kevin and Zou, Andy},
  title        = {The Language Model Evaluation Harness},
  month        = 07,
  year         = 2024,
  publisher    = {Zenodo},
  version      = {v0.4.3},
  doi          = {10.5281/zenodo.12608602},
  url          = {https://zenodo.org/records/12608602}
}

@inproceedings{dubois2024length,
  title={Length-Controlled {AlpacaEval}: A Simple Way to Debias Automatic Evaluators},
  author={Dubois, Yann and Galambosi, Bal{\'a}zs and Liang, Percy and Hashimoto, Tatsunori B.},
  booktitle={Conference on Language Modeling},
  year={2024},
  url={https://openreview.net/forum?id=CybBmzWBX0}
}

@article{kulesza2012determinantal,
  title={Determinantal Point Processes for Machine Learning},
  author={Kulesza, Alex and Taskar, Ben},
  journal={Foundations and Trends{\textregistered} in Machine Learning},
  volume={5},
  number={2--3},
  pages={123--286},
  year={2012},
  publisher={Now Publishers, Inc.},
  doi={10.1561/2200000044}
}

@inproceedings{khan2026dfrag,
  title={{DF-RAG}: Query-Aware Diversity for Retrieval-Augmented Generation},
  author={Khan, Saadat Hasan and Hong, Spencer and Wu, Jingyu and Lybarger, Kevin and Yin, Youbing and Babinsky, Erin and Liu, Daben},
  booktitle={Findings of the Association for Computational Linguistics: EACL 2026},
  pages={2873--2894},
  year={2026},
  publisher={Association for Computational Linguistics},
  doi={10.18653/v1/2026.findings-eacl.150},
  url={https://aclanthology.org/2026.findings-eacl.150/}
}

@misc{li2024dmqrrag,
  title={{DMQR-RAG}: Diverse Multi-Query Rewriting for {RAG}},
  author={Li, Zhicong and Wang, Jiahao and Jiang, Zhishu and Mao, Hangyu and Chen, Zhongxia and Du, Jiazhen and Zhang, Yuanxing and Zhang, Fuzheng and Zhang, Di and Liu, Yong},
  year={2024},
  eprint={2411.13154},
  archivePrefix={arXiv},
  primaryClass={cs.IR},
  doi={10.48550/arXiv.2411.13154},
  url={https://arxiv.org/abs/2411.13154}
}

@misc{sentenceTransformersMultiQaMpnet,
  title={{Sentence-Transformers} {Multi-QA} {MPNet} Base Cosine Model},
  author={{Sentence-Transformers}},
  year={2021},
  howpublished={Hugging Face model repository},
  url={https://huggingface.co/sentence-transformers/multi-qa-mpnet-base-cos-v1}
}

@inproceedings{reimers2019sentence,
  title={{Sentence-BERT}: Sentence Embeddings using Siamese {BERT}-Networks},
  author={Reimers, Nils and Gurevych, Iryna},
  booktitle={Proceedings of the 2019 Conference on Empirical Methods in Natural Language Processing and the 9th International Joint Conference on Natural Language Processing},
  pages={3982--3992},
  year={2019},
  publisher={Association for Computational Linguistics},
  doi={10.18653/v1/D19-1410},
  url={https://aclanthology.org/D19-1410/}
}

@inproceedings{song2020mpnet,
  title={{MPNet}: Masked and Permuted Pre-training for Language Understanding},
  author={Song, Kaitao and Tan, Xu and Qin, Tao and Lu, Jianfeng and Liu, Tie-Yan},
  booktitle={Advances in Neural Information Processing Systems},
  volume={33},
  pages={16857--16867},
  year={2020}
}

@inproceedings{deb2025fishersft,
  title={{FisherSFT}: Data-Efficient Supervised Fine-Tuning of Language Models Using Information Gain},
  author={Deb, Rohan and Thekumparampil, Kiran Koshy and Kalantari, Kousha and Hiranandani, Gaurush and Sabach, Shoham and Kveton, Branislav},
  booktitle={Proceedings of the 42nd International Conference on Machine Learning},
  volume={267},
  series={Proceedings of Machine Learning Research},
  pages={12924--12943},
  year={2025},
  publisher={PMLR},
  url={https://proceedings.mlr.press/v267/deb25a.html}
}

@inproceedings{yang2018hotpotqa,
  title={{HotpotQA}: A Dataset for Diverse, Explainable Multi-hop Question Answering},
  author={Yang, Zhilin and Qi, Peng and Zhang, Saizheng and Bengio, Yoshua and Cohen, William W. and Salakhutdinov, Ruslan and Manning, Christopher D.},
  booktitle={Proceedings of the 2018 Conference on Empirical Methods in Natural Language Processing},
  pages={2369--2380},
  year={2018},
  publisher={Association for Computational Linguistics},
  doi={10.18653/v1/D18-1259},
  url={https://aclanthology.org/D18-1259/}
}

@inproceedings{ho2020constructing,
  title={Constructing A Multi-hop {QA} Dataset for Comprehensive Evaluation of Reasoning Steps},
  author={Ho, Xanh and Duong Nguyen, Anh-Khoa and Sugawara, Saku and Aizawa, Akiko},
  booktitle={Proceedings of the 28th International Conference on Computational Linguistics},
  pages={6609--6625},
  year={2020},
  publisher={International Committee on Computational Linguistics},
  doi={10.18653/v1/2020.coling-main.580},
  url={https://aclanthology.org/2020.coling-main.580/}
}

@article{trivedi2022musique,
  title={{MuSiQue}: Multihop Questions via Single-hop Question Composition},
  author={Trivedi, Harsh and Balasubramanian, Niranjan and Khot, Tushar and Sabharwal, Ashish},
  journal={Transactions of the Association for Computational Linguistics},
  volume={10},
  pages={539--554},
  year={2022},
  publisher={MIT Press},
  doi={10.1162/tacl_a_00475},
  url={https://aclanthology.org/2022.tacl-1.31/}
}

@inproceedings{bai2023longbench,
  title={{LongBench}: A Bilingual, Multitask Benchmark for Long Context Understanding},
  author={Bai, Yushi and Lv, Xin and Zhang, Jiajie and Lyu, Hongchang and Tang, Jiankai and Huang, Zhidian and Du, Zhengxiao and Liu, Xiao and Zeng, Aohan and Hou, Lei and Dong, Yuxiao and Tang, Jie and Li, Juanzi},
  booktitle={Proceedings of the 62nd Annual Meeting of the Association for Computational Linguistics (Volume 1: Long Papers)},
  pages={3119--3137},
  year={2024},
  publisher={Association for Computational Linguistics},
  doi={10.18653/v1/2024.acl-long.172},
  url={https://aclanthology.org/2024.acl-long.172/}
}

@inproceedings{carbonell1998use,
  title={The Use of {MMR}, Diversity-Based Reranking for Reordering Documents and Producing Summaries},
  author={Carbonell, Jaime and Goldstein, Jade},
  booktitle={Proceedings of the 21st Annual International ACM SIGIR Conference on Research and Development in Information Retrieval},
  pages={335--336},
  year={1998},
  doi={10.1145/290941.291025}
}

@inproceedings{santos2010exploiting,
  title={Exploiting Query Reformulations for Web Search Result Diversification},
  author={Santos, Rodrygo L. T. and Macdonald, Craig and Ounis, Iadh},
  booktitle={Proceedings of the 19th International Conference on World Wide Web},
  pages={881--890},
  year={2010},
  doi={10.1145/1772690.1772780}
}

@inproceedings{chen2025mig,
  title={{MIG}: Automatic Data Selection for Instruction Tuning by Maximizing Information Gain in Semantic Space},
  author={Chen, Yicheng and Li, Yining and Hu, Kai and Ma, Zerun and Ye, Haochen and Chen, Kai},
  booktitle={Findings of the Association for Computational Linguistics: ACL 2025},
  pages={9902--9915},
  year={2025},
  publisher={Association for Computational Linguistics},
  doi={10.18653/v1/2025.findings-acl.515},
  url={https://aclanthology.org/2025.findings-acl.515/}
}

@inproceedings{zhang2025grape,
  title={The Best Instruction-Tuning Data are Those That Fit},
  author={Zhang, Dylan and Dai, Qirun and Peng, Hao},
  booktitle={Advances in Neural Information Processing Systems},
  volume={38},
  year={2025},
  doi={10.52202/085713-4717},
  url={https://proceedings.neurips.cc/paper_files/paper/2025/hash/cea04322465ad2f261f08e5b47ba9e7a-Abstract-Conference.html}
}

@inproceedings{li2024datacomp,
  title={{DataComp-LM}: In Search of the Next Generation of Training Sets for Language Models},
  author={Li, Jeffrey and Fang, Alex and Smyrnis, Georgios and Ivgi, Maor and Jordan, Matt and Gadre, Samir and Bansal, Hritik and Guha, Etash and Keh, Sedrick and Arora, Kushal and Garg, Saurabh and Xin, Rui and Muennighoff, Niklas and Heckel, Reinhard and Mercat, Jean and Chen, Mayee and Gururangan, Suchin and Wortsman, Mitchell and Albalak, Alon and Bitton, Yonatan and Nezhurina, Marianna and Abbas, Amro and Hsieh, Cheng-Yu and Ghosh, Dhruba and Gardner, Josh and Kilian, Maciej and Zhang, Hanlin and Shao, Rulin and Pratt, Sarah and Sanyal, Sunny and Ilharco, Gabriel and Daras, Giannis and Marathe, Kalyani and Gokaslan, Aaron and Zhang, Jieyu and Chandu, Khyathi and Nguyen, Thao and Vasiljevic, Igor and Kakade, Sham and Song, Shuran and Sanghavi, Sujay and Faghri, Fartash and Oh, Sewoong and Zettlemoyer, Luke and Lo, Kyle and El-Nouby, Alaaeldin and Pouransari, Hadi and Toshev, Alexander and Wang, Stephanie and Groeneveld, Dirk and Soldaini, Luca and Koh, Pang Wei and Jitsev, Jenia and Kollar, Thomas and Dimakis, Alexandros G. and Carmon, Yair and Dave, Achal and Schmidt, Ludwig and Shankar, Vaishaal},
  booktitle={Advances in Neural Information Processing Systems},
  volume={37},
  year={2024},
  doi={10.52202/079017-0455},
  url={https://proceedings.neurips.cc/paper_files/paper/2024/hash/19e4ea30dded58259665db375885e412-Abstract-Datasets_and_Benchmarks_Track.html}
}

@inproceedings{black2022gptneox,
  title={{GPT-NeoX-20B}: An Open-Source Autoregressive Language Model},
  author={Black, Sidney and Biderman, Stella and Hallahan, Eric and Anthony, Quentin and Gao, Leo and Golding, Laurence and He, Horace and Leahy, Connor and McDonell, Kyle and Phang, Jason and Pieler, Michael and Prashanth, Usvsn Sai and Purohit, Shivanshu and Reynolds, Laria and Tow, Jonathan and Wang, Ben and Weinbach, Samuel},
  booktitle={Proceedings of BigScience Episode 5: Workshop on Challenges and Perspectives in Creating Large Language Models},
  pages={95--136},
  year={2022},
  publisher={Association for Computational Linguistics},
  doi={10.18653/v1/2022.bigscience-1.9},
  url={https://aclanthology.org/2022.bigscience-1.9/}
}

@inproceedings{zellers2019hellaswag,
  title={{HellaSwag}: Can a Machine Really Finish Your Sentence?},
  author={Zellers, Rowan and Holtzman, Ari and Bisk, Yonatan and Farhadi, Ali and Choi, Yejin},
  booktitle={Proceedings of the 57th Annual Meeting of the Association for Computational Linguistics},
  pages={4791--4800},
  year={2019},
  publisher={Association for Computational Linguistics},
  doi={10.18653/v1/P19-1472},
  url={https://aclanthology.org/P19-1472/}
}

@inproceedings{bisk2020piqa,
  title={{PIQA}: Reasoning about Physical Commonsense in Natural Language},
  author={Bisk, Yonatan and Zellers, Rowan and {Le Bras}, Ronan and Gao, Jianfeng and Choi, Yejin},
  booktitle={Proceedings of the AAAI Conference on Artificial Intelligence},
  volume={34},
  pages={7432--7439},
  year={2020},
  doi={10.1609/aaai.v34i05.6239}
}

@article{clark2018think,
  title={Think You Have Solved Question Answering? Try {ARC}, the {AI2} Reasoning Challenge},
  author={Clark, Peter and Cowhey, Isaac and Etzioni, Oren and Khot, Tushar and Sabharwal, Ashish and Schoenick, Carissa and Tafjord, Oyvind},
  journal={arXiv preprint arXiv:1803.05457},
  year={2018}
}

@inproceedings{sakaguchi2020winogrande,
  title={{WinoGrande}: An Adversarial Winograd Schema Challenge at Scale},
  author={Sakaguchi, Keisuke and {Le Bras}, Ronan and Bhagavatula, Chandra and Choi, Yejin},
  booktitle={Proceedings of the AAAI Conference on Artificial Intelligence},
  volume={34},
  pages={8732--8740},
  year={2020},
  doi={10.1609/aaai.v34i05.6399}
}

@inproceedings{mihaylov2018openbookqa,
  title={Can a Suit of Armor Conduct Electricity? A New Dataset for Open Book Question Answering},
  author={Mihaylov, Todor and Clark, Peter and Khot, Tushar and Sabharwal, Ashish},
  booktitle={Proceedings of the 2018 Conference on Empirical Methods in Natural Language Processing},
  pages={2381--2391},
  year={2018},
  publisher={Association for Computational Linguistics},
  doi={10.18653/v1/D18-1260},
  url={https://aclanthology.org/D18-1260/}
}

@inproceedings{welbl2017crowdsourcing,
  title={Crowdsourcing Multiple Choice Science Questions},
  author={Welbl, Johannes and Liu, Nelson F. and Gardner, Matt},
  booktitle={Proceedings of the 3rd Workshop on Noisy User-generated Text},
  pages={94--106},
  year={2017},
  publisher={Association for Computational Linguistics},
  doi={10.18653/v1/W17-4413},
  url={https://aclanthology.org/W17-4413/}
}

@misc{open_lm,
  author = {Gururangan, Suchin and Wortsman, Mitchell and Gadre, Samir Yitzhak and Dave, Achal and Kilian, Maciej and Shi, Weijia and Mercat, Jean and Smyrnis, Georgios and Ilharco, Gabriel and Jordan, Matt and Heckel, Reinhard and Dimakis, Alex and Farhadi, Ali and Shankar, Vaishaal and Schmidt, Ludwig},
  title = {{open\_lm}: A Minimal but Performative Language Modeling Repository},
  year = {2023},
  note = {GitHub repository},
  url = {https://github.com/mlfoundations/open_lm/}
}

@misc{wikimedia2023wikipedia,
  author = {{Wikimedia Foundation}},
  title = {{wikimedia/wikipedia}: {English Wikipedia} 20231101.en},
  year = {2023},
  note = {Hugging Face dataset configuration 20231101.en},
  url = {https://huggingface.co/datasets/wikimedia/wikipedia}
}

@inproceedings{loshchilov2019decoupled,
  title = {Decoupled Weight Decay Regularization},
  author = {Loshchilov, Ilya and Hutter, Frank},
  booktitle = {International Conference on Learning Representations},
  year = {2019},
  url = {https://openreview.net/forum?id=Bkg6RiCqY7}
}

@misc{chai2022fast,
  title={{FAST}: Improving Controllability for Text Generation with Feedback Aware Self-Training},
  author={Chai, Junyi and Pryzant, Reid and Dong, Victor Ye and Golobokov, Konstantin and Zhu, Chenguang and Liu, Yi},
  year={2022},
  eprint={2210.03167},
  archivePrefix={arXiv},
  primaryClass={cs.CL},
  doi={10.48550/arXiv.2210.03167},
  url={https://arxiv.org/abs/2210.03167}
}

@inproceedings{golobokov2022deepgen,
  title={{DeepGen}: Diverse Search Ad Generation and Real-Time Customization},
  author={Golobokov, Konstantin and Chai, Junyi and Dong, Victor Ye and Gu, Mandy and Chi, Bingyu and Cao, Jie and Yan, Yulan and Liu, Yi},
  booktitle={Proceedings of the 2022 Conference on Empirical Methods in Natural Language Processing: System Demonstrations},
  month={December},
  year={2022},
  address={Abu Dhabi, UAE},
  publisher={Association for Computational Linguistics},
  pages={191--199},
  doi={10.18653/v1/2022.emnlp-demos.19},
  url={https://aclanthology.org/2022.emnlp-demos.19/}
}

@inproceedings{chen2023semioffline,
  title={Semi-Offline Reinforcement Learning for Optimized Text Generation},
  author={Chen, Changyu and Wang, Xiting and Jin, Yiqiao and Dong, Victor Ye and Dong, Li and Cao, Jie and Liu, Yi and Yan, Rui},
  booktitle={Proceedings of the 40th International Conference on Machine Learning},
  pages={5087--5103},
  year={2023},
  editor={Krause, Andreas and Brunskill, Emma and Cho, Kyunghyun and Engelhardt, Barbara and Sabato, Sivan and Scarlett, Jonathan},
  volume={202},
  series={Proceedings of Machine Learning Research},
  publisher={PMLR},
  url={https://proceedings.mlr.press/v202/chen23ad.html}
}

@inproceedings{ouyang2022training,
  title={Training Language Models to Follow Instructions with Human Feedback},
  author={Ouyang, Long and Wu, Jeff and Jiang, Xu and Almeida, Diogo and Wainwright, Carroll L. and Mishkin, Pamela and Zhang, Chong and Agarwal, Sandhini and Slama, Katarina and Ray, Alex and Schulman, John and Hilton, Jacob and Kelton, Fraser and Miller, Luke and Simens, Maddie and Askell, Amanda and Welinder, Peter and Christiano, Paul and Leike, Jan and Lowe, Ryan},
  booktitle={Advances in Neural Information Processing Systems},
  volume={35},
  year={2022},
  url={https://proceedings.neurips.cc/paper_files/paper/2022/hash/b1efde53be364a73914f58805a001731-Abstract.html}
}

@inproceedings{rafailov2023direct,
  title={Direct Preference Optimization: Your Language Model is Secretly a Reward Model},
  author={Rafailov, Rafael and Sharma, Archit and Mitchell, Eric and Manning, Christopher D. and Ermon, Stefano and Finn, Chelsea},
  booktitle={Advances in Neural Information Processing Systems},
  volume={36},
  year={2023},
  doi={10.52202/075280-2338},
  url={https://proceedings.neurips.cc/paper_files/paper/2023/hash/a85b405ed65c6477a4fe8302b5e06ce7-Abstract-Conference.html}
}
